\documentclass{article}

\usepackage{arxiv}

\usepackage[utf8]{inputenc} % allow utf-8 input
\usepackage[T1]{fontenc}    % use 8-bit T1 fonts
\usepackage{hyperref}       % hyperlinks
\usepackage{url}            % simple URL typesetting
\usepackage{booktabs}       % professional-quality tables
\usepackage{amsfonts}       % blackboard math symbols
\usepackage{nicefrac}       % compact symbols for 1/2, etc.
\usepackage{microtype}      % microtypography
\usepackage{lipsum}		% Can be removed after putting your text content
\usepackage{graphicx}
\usepackage{natbib}
\usepackage{doi}

\DeclareMathAlphabet\mathbb{U}{msb}{m}{n}
\usepackage{enumitem}
\usepackage[vlined,ruled,linesnumbered]{algorithm2e}
\usepackage{nth}
\graphicspath{{Fig/}}
\usepackage{hyperref}
\usepackage{xcolor}

\usepackage[most]{tcolorbox}
\newtcolorbox{highlighted}{colback=yellow,coltext=black,breakable}
\usepackage{fdsymbol}
\usepackage[leqno]{mathtools}
\usepackage{graphicx}
\usepackage{color}
\usepackage{amsmath,amsfonts}
\usepackage{mathrsfs}
\usepackage{hyperref}
\usepackage{wrapfig}
\usepackage{listings}
\usepackage{ifpdf}
\hypersetup{pdfauthor={Name}}
\usepackage{subfig}
\usepackage{booktabs}
\newtheorem{definition}{Definition}
\newtheorem{example}{Example}
\newtheorem{remark}{Remark}
\newtheorem{lemma}{Lemma}
\newtheorem{proof}{Proof}

\newcommand{\val}{\nu}

\newcommand{\F}{\mathbf{F}}
\newcommand{\true}{\mathit{true}}

\newcommand{\Reals}{\mathbb{R}}

\newcommand{\params}{\mathcal{P}}

\newcommand{\mypara}[1]{\vspace{0.3em} \noindent {\bf #1.\ }}

\newcommand{\eqdef}{\mathrel{\stackrel{\makebox[0pt]{\mbox{\normalfont\tiny def}}}{=}}}

\newcommand{\proj}{\pi}

\newcommand{\projlex}{\proj_{\mathrm{lex}}}

\newcommand{\numclusters}{\mathit{numC}}

\newcommand{\T}{true}

\newcommand{\ev}[1]{\mathbf{F}_{#1}}
\newcommand{\glob}[1]{\mathbf{G}_{#1}}
\newcommand{\until}[1]{\mathrm{U}_{#1}}

\newcommand{\somewhere}[1]{\diamonddiamond_{#1}}
\newcommand{\everywhere}[1]{\boxbox_{#1}}
\newcommand{\surround}[1]{\circledcirc_{#1}}
\newcommand{\reach}[1]{\mathcal{R}_{#1}}
\newcommand{\escape}[1]{\mathcal{E}_{#1}}

\newcommand{\wfun}{W}

\newcommand{\route}{\tau}

\newcommand{\sts}{\sigma}
\newcommand{\traces}{\Sigma}

\newcommand{\dhaver}{d_\mathrm{hvrsn}}

\newcommand{\spatialmodel}{\mathcal{S}}
\newcommand{\locations}{L}
\newcommand{\reals}{\mathbb{R}}
\newcommand{\routes}{\mathcal{T}}
\newcommand{\routedistance}{d_{\spatialmodel}}

\newcommand{\timedomain}{\mathbb{T}}
\newcommand{\valuedomain}{\mathcal{V}}

\newcommand{\bikes}{B}
\newcommand{\slots}{S}

\newcommand{\distint}{D}

\newcommand{\timeparams}{\mathcal{P}_{\timedomain}}
\newcommand{\valueparams}{\mathcal{P}_{\valuedomain}}
\newcommand{\spaceparams}{\mathcal{P}_{\routedistance}}

\newcommand{\param}{\mathbf{p}}
\newcommand{\polarity}{\gamma}
\newcommand{\paramspace}{P}

\newcommand{\validitydomain}{V}
\newcommand{\validitydomainboundary}{\partial \validitydomain}

\newcommand{\preclex}{\prec_\mathrm{lex}}

\newcommand{\clusterlabel}{C}

\newcommand{\lfun}{\mathcal{L}}
\newcommand{\fwait}{\varphi_\mathrm{wait}}
\newcommand{\fwalk}{\varphi_\mathrm{walk}}

\title{Mining Interpretable Spatio-temporal Logic Properties for Spatially Distributed Systems}

\author{
  Sara Mohammadinejad\\
  University of Southern California\\
  \texttt{saramoha@usc.edu} \\
   \And
  Jyotirmoy V. Deshmukh*\\
  University of Southern California\\
  \texttt{jdeshmuk@usc.edu} \\
   \And
  Laura Nenzi*\\
  University of Trieste\\
  \texttt{lnenzi@units.it} 
}

\renewcommand{\shorttitle}{Mining Spatio-temporal Logic Formulas}

\hypersetup{
pdftitle={Mining Interpretable Spatio-temporal Logic Properties for Spatially Distributed Systems},
pdfsubject={q-bio.NC, q-bio.QM},
pdfauthor={Sara Mohammadinejad, Jyotirmoy V. Deshmukh, Laura Nenzi},
pdfkeywords={Distributed systems, Unsupervised learning, Spatio-temporal data, Interpretability, Spatio-temporal reach and escape logic.}
}

\begin{document}
\maketitle

\begin{abstract}
The Internet-of-Things, complex sensor networks, multi-agent cyber-physical
systems are all examples of spatially distributed systems that continuously
evolve in time. Such systems generate huge amounts of spatio-temporal data,
and system designers are often interested in analyzing and discovering
structure within the data. There has been considerable interest in learning
causal and logical properties of temporal data using logics such as Signal
Temporal Logic (STL); however, there is limited work on discovering such
relations on {\em spatio}-{\em temporal} data. We propose the first set of
algorithms for {\em unsupervised learning} for spatio-temporal data. Our
method does automatic feature extraction from the spatio-temporal data by
projecting it onto the parameter space of a {\em parametric spatio-temporal
reach and escape logic} (PSTREL). We propose an agglomerative hierarchical
clustering technique that guarantees that each cluster satisfies a distinct
STREL formula. We show that our method generates STREL formulas of bounded
description complexity using a novel decision-tree approach which
generalizes previous unsupervised learning techniques for Signal Temporal
Logic. We demonstrate the effectiveness of our approach on case studies from
diverse domains such as urban transportation, epidemiology, green
infrastructure, and air quality monitoring.

\end{abstract}

\let\thefootnote\relax\footnotetext{*: Equal contribution}

% keywords can be removed

\section{Introduction}
Due to rapid improvements in sensing and communication technologies,
embedded systems are now often spatially distributed. Such spatially
distributed systems (SDS) consist of heterogeneous components embedded
in a specific topological space, whose time-varying behaviors evolve
according to complex mutual inter-dependence relations
\cite{nenzi2018qualitative}.  In the formal methods community,
tremendous advances have been achieved for verification and analysis
of distributed systems. However, most
formal techniques abstract away the specific spatial aspects of
distributed systems, which can be of crucial importance in certain
applications.  For example, consider the problem of developing a
bike-sharing system (BSS) in a ``sharing economy.'' Here, the system
consists of a number of bike stations that would use sensors to detect
the number of bikes present at a station, and use incentives to let
users return bikes to stations that are running low. The bike stations
themselves could be arbitrary locations in a city, and the design of
an effective BSS would require reasoning about the distance to nearby
locations, and the time-varying demand or supply at each location. For instance, the property ``there is always a bike and a slot available at distance d from a bike station'' depends on the distance of the bike station to its nearby stations. Evaluating whether the BSS functions correctly is a verification
problem where the specification is a {\em spatio}-{\em temporal} logic
formula.  Similarly, consider the problem of coordinating the
movements of multiple mobile robots, or a HVAC controller that
activates heating or cooling in parts of a building based on
occupancy. Given spatio-temporal execution traces of nodes in such
systems, we may be interested in analyzing the data to solve several
classical formal methods problems such as fault localization,
debugging, invariant generation or specification mining.  It is
increasingly urgent to formulate methods that enable reasoning about
spatially-distributed systems in a way that explicitly incorporates
their spatial topology.

In this paper, we focus on one specific aspect of spatio-temporal
reasoning: mining interpretable logical properties from data in an
SDS.  We model a SDS as a directed or undirected graph where
individual compute nodes are vertices, and edges model either the
connection topology or spatial proximity. In the past, analytic models
based on partial differential equations (e.g. diffusion equations)
\cite{fiedler2003spatio} have been used to express the spatio-temporal evolution of
these systems. While such formalisms are incredibly powerful, they are
also quite difficult to interpret.  Traditional machine learning (ML)
approaches have also been used to uncover the structure of such
spatio-temporal systems, but these techniques also suffer
from the lack of interpretability. Our proposed method draws on a
recently proposed logic known as {\em Spatio}-{\em Temporal Reach and
Escape Logic} (STREL) \cite{bartocci2017monitoring}. Recent research
on STREL has focused on efficient algorithms for runtime verification
and monitoring of STREL specifications
\cite{bartocci2017monitoring,bartocci2020moonlight}.  However, there
is no existing work on mining STREL specifications.  

Mined STREL specifications can be useful in many different contexts in
the design of spatially distributed systems; an incomplete list of
usage scenarios includes the following applications: (1) Mined STREL
formulas can serve as spatio-temporal invariants that are satisfied by
the computing nodes, (2) STREL formulas could be used by developers to
characterize the properties of a deployed spatially distributed
system, which can then be used to monitor any subsequent updates to
the system, (3) Clustering nodes that satisfy similar STREL formulas
can help debug possible bottlenecks and violations in communication
protocols in such distributed systems.

There is considerable amount of recent work on learning temporal logic
formulas from data
\cite{vazquez2017logical,jin2015mining,mohammadinejad2020interpretable,mohammadinejad2020mining}. In
particular, the work in this paper is closest to the work on
unsupervised clustering of time-series data using Signal Temporal
Logic \cite{vazquez2017logical}. In this work, the authors assume that the user
provides a Parametric Signal Temporal Logic (PSTL) formula, and the
procedure projects given temporal data onto the parameter domain of
the PSTL formula. The authors use off-the-shelf clustering techniques
to group parameter values and identify STL formulas corresponding to
each cluster. There are a few hurdles in applying such an approach to
spatio-temporal data. First, in \cite{vazquez2017logical}, the authors assume a
monotonic fragment of PSTL: there is no such fragment identified in
the literature for STREL. Second, in \cite{vazquez2017logical}, the authors
assume that clusters in the parameter space can be separated by
axis-aligned hyper-boxes. Third, given spatio-temporal data, we can
have different choices to impose the edge relation on nodes, which can
affect the formula we learn.

To address the shortcomings of previous techniques, we introduce
PSTREL, by treating threshold constants in signal predicates, time
bounds in temporal operators, and distance bounds in spatial operators
as parameters. We then identify a monotonic fragment of PSTREL, and
propose a multi-dimensional binary-search based procedure to infer
{\em tight} parameter valuations for the given PSTREL formula. We also
explore the space of implied edge relations between spatial nodes, proposing an algorithm to define the most suitable graph.
%showing the impact of the number of nodes on computational times in the learning procedures.
After defining a projection operator that maps a given spatio-temporal
signal to parameter values of the given PSTREL formula, we use an
agglomerative hierarchical clustering technique to cluster spatial
locations into hyperboxes. We improve the method of \cite{vazquez2017logical} by
introducing a decision-tree based approach to systematically split
overlapping hyperbox clusters. The result of our method produces
axis-aligned hyperbox clusters that can be compactly described by an
STREL formula that has length proportional to the number of parameters
in the given PSTREL formula (and independent of the number of
clusters). Finally, we give human-interpretable meanings for each
cluster. We show the usefulness of our approach considering four benchmarks: COVID-19 data from LA County, Outdoor air quality data, BSS data and movements of the customer in a Food Court.

\noindent \paragraph{\bf Running Example: A Bike Sharing System (BSS)}
To ease the exposition of key ideas in the paper, we use an example of
a BSS deployed in the city of Edinburgh, UK. The BSS consists of a
number of bike stations, distributed over a geographic area. Each
station has a fixed number of bike slots.  Users can pick up a bike,
use it for a while, and then return it to another station in the area.
The data that we analyze are the number of bikes (B) and empty slots
(S) at each time step in each bike station. With the advent of
electric bikes, BSS have become an important aspect in urban mobility,
and such systems make use of embedded devices for diverse purposes
such as tracking bike usage, billing, and displaying information about
availability to users over apps. Fig.~\ref{fig:bss_map_runex} shows the map of the Edinburgh city with the bike stations. Different colors of the nodes represent different learned clusters as can be seen in Fig.~\ref{fig:bss_clustering_runex}. For example, using our approach, we learn that stations in {\em orange} cluster have a long wait time, and stations in {\em red} cluster are the most undesirable stations as they have long wait time and do not have nearby stations with bike availability. If we look at the actual location of {\em red} points in Fig. 1b, they are indeed far away stations.

\begin{figure}[t]
\centering
% \setkeys{Gin}{height=3cm} 
% \makebox[\textwidth]{%
%   \setlength{\tabcolsep}{3pt}%
%   \begin{tabular}{@{}cc@{}}
    \subfloat[Clusters learned\label{fig:bss_clustering_runex}]{\includegraphics[width =
    .45\textwidth]{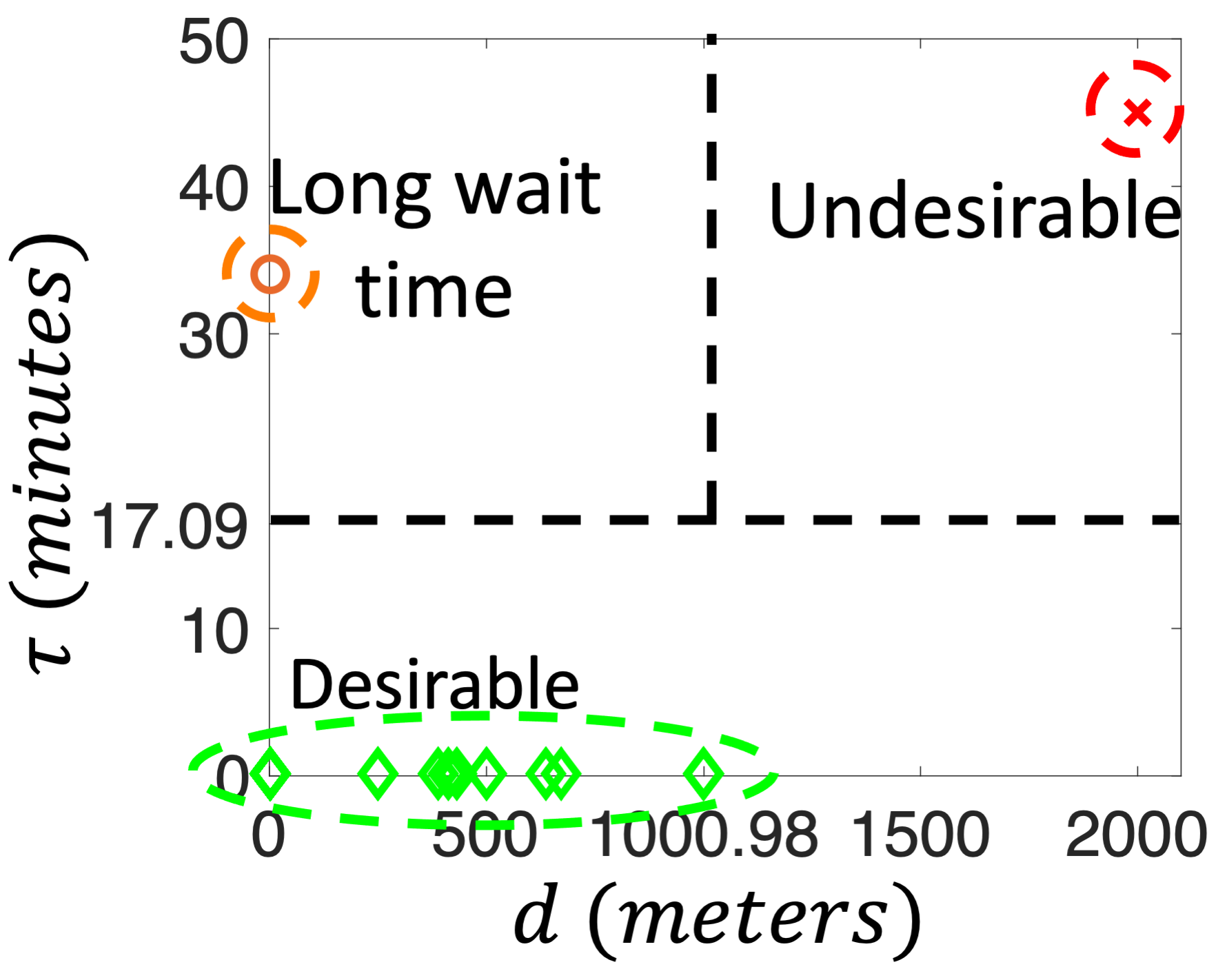}} \qquad
  \subfloat[BSS locations in Edinburgh\label{fig:bss_map_runex}]{\includegraphics[width = .45\textwidth]{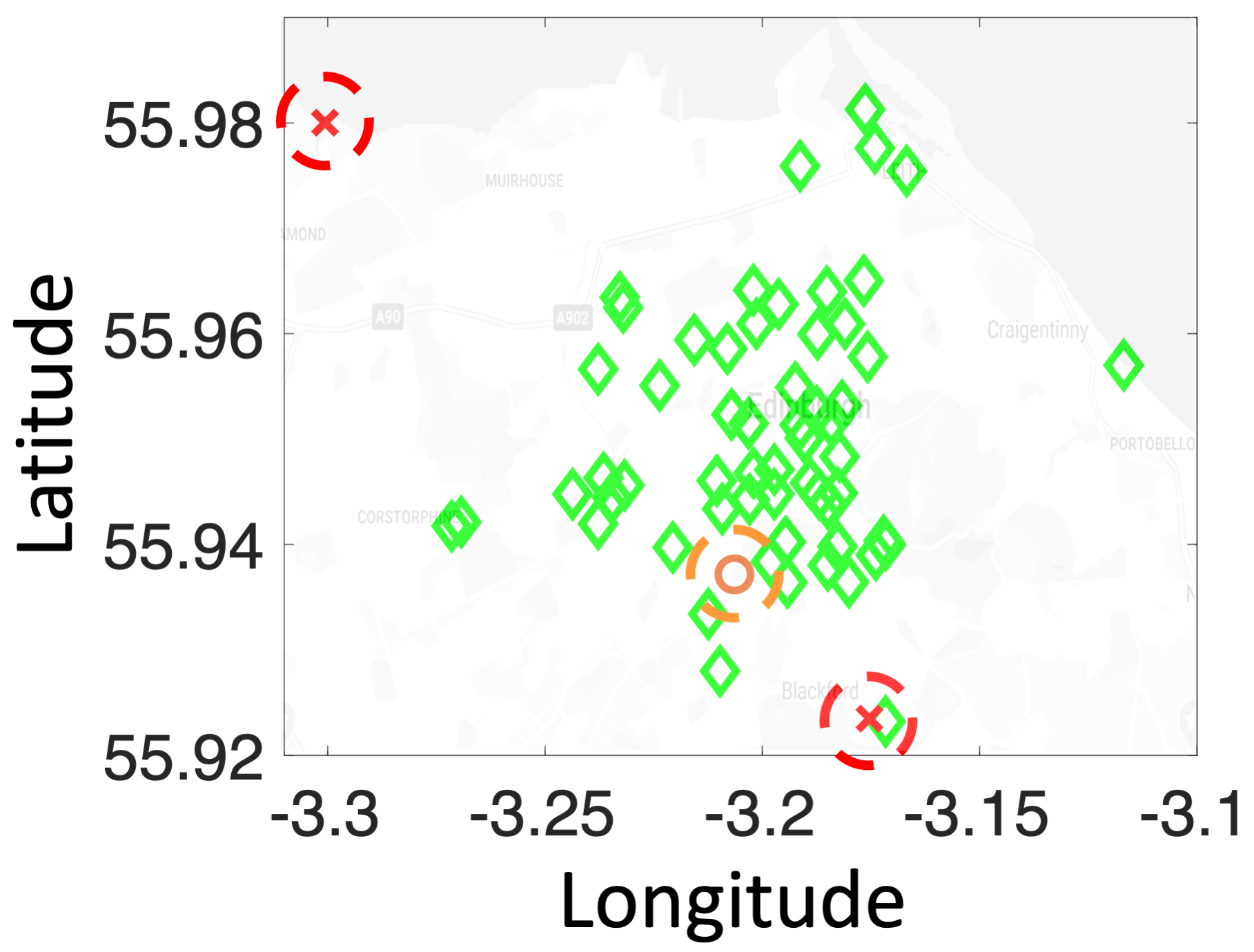}} 
%   \end{tabular}%
% }
\caption{Interpretable clusters automatically identified by our technique.}
\label{fig:bss_example_intro}
\end{figure}

\section{Background}
\label{sec:background}
% <<<<<<< HEAD
In this section, we introduce the notation and terminology for spatial models and
spatio-temporal traces and we describe Spatio-Temporal Reach and Escape Logic (STREL). 

\begin{definition}[Spatial Model]
\label{def:spatmodel}
A spatial model $\spatialmodel$ is defined as a pair $\langle
\locations, \wfun \rangle$, where $\locations$ is a set of nodes or
locations and $\wfun \subseteq \locations \times \reals \times
\locations$ is a nonempty relation associating each distinct pair
$\ell_1, \ell_2 \in \locations$ with a label $w \in \Reals$ (also
denoted $\ell_1 \xrightarrow{w} \ell_2$).
\end{definition}
% If
% $(\ell_1, \mathnormal{w}, \ell_2) \in \wfun$, it means that there is
% an edge from $\ell_1$ to $\ell_2$ with weight $\mathnormal{w} \in
% \mathbb{R}$, i.e.  $\ell_1 \xrightarrow{w} \ell_2$. 
% introduce Haversine formula here
% Introduce the notion of a topological space
There are many different choices possible for the proximity
relation $\wfun$; for example, $\wfun$ could be defined in a way that
the edge-weights indicate spatial proximity, communication network
connectivity etc. Given a set of locations, unless there is a
user-specified $\wfun$, we note that there are several graphs (and
associated edge-weights) that we can use to express spatial models.
We explore these possibilities in Sec.~\ref{sec:graph}. For the rest
of this section, we assume that $\wfun$ is defined using the notion of
$(\delta,d)$-connectivity graph as defined in Definition~\ref{def:deltaconngraph}.

\newcommand{\metricspace}{M}

\begin{definition}[$(\delta,d)$-connectivity spatial model]
\label{def:deltaconngraph}
Given a compact metric space $\metricspace$ with the distance metric
$d:\metricspace\times\metricspace \to \Reals^{\ge 0}$, a set of
locations $\locations$ that is a finite subset of $\metricspace$, and
a fixed $\delta \in \Reals, \delta > 0$, a $(\delta,d)$-connectivity
spatial model is defined as $\langle\locations,\wfun\rangle$, where 
$(\ell_1,w,\ell_2) \in \wfun$ iff $d(\ell_1,\ell_2) = w$, and $w < \delta$. 
\end{definition}

\begin{example} In the BSS, each bike station is a node/location in
the spatial model, where locations are assumed to lie on the metric
space defined by the 3D spherical manifold of the earth's surface;
each location is defined by its latitude and longitude, and the
distance metric is the {\em Haversine distance}\footnote{Haversine Formula gives minimum distance between any two points on sphere by using their latitudes and longitudes.}.
Fig.~\ref{fig:delta_graph} shows the $\delta$-connectivity
graph of the Edinburgh BSS, with $\delta=1 km$.
\end{example}

%%%%%%%%%%%%%%%%%%%%%%%%%%%%%%%%%%%%%%%%%%%%%%
%   Route, Route Distance, Location Distance
%%%%%%%%%%%%%%%%%%%%%%%%%%%%%%%%%%%%%%%%%%%%%%

\begin{definition}[Route]
For a spatial model $\spatialmodel =\langle L, \wfun \rangle$, a route
$\route$ is an infinite sequence $\ell_0\ell_1\cdots\ell_k\cdots$ such
that for any $i \geq 0$, $\ell_i \xrightarrow{w_i} \ell_{i+1}$. 
\end{definition}

\noindent For a route $\route$, $\route[i]$ denotes the $i^{th}$ node $\ell_i$
in $\route$, $\route[i..]$ indicates the suffix route
$\ell_i\ell_{i+1}...$, and $\route(\ell)$ denotes $\min{i\mid
\route[i]=\ell}$, i.e. the first occurrence of $\ell$ in $\route$.
Note that $\route(\ell) = \infty$ if $\forall i \route[i]\neq \ell$.
We use $\routes(\spatialmodel)$ to denote the set of routes in 
$\spatialmodel$, and $\routes(\spatialmodel,\ell)$ to denote the set 
of routes in $\spatialmodel$ starting from $\ell \in \locations$. We 
can use routes to define the route distance between two locations in the 
spatial model as follows.

\begin{definition}[Route Distance and Spatial Model Induced Distance]
\label{eq:routdist}
Given a route $\route$, the route distance along $\tau$ up to a
location $\ell$ denoted $\routedistance^\tau(\ell)$ is defined as
$\sum_{i=0}^{\route(\ell)} w_i$.  The spatial model induced distance
between locations $\ell_1$ and $\ell_2$ (denoted
$\routedistance(\ell_1,\ell_2)$) is defined as:
$\routedistance(\ell_1,\ell_2) = \min_{\tau \in
\routes(\spatialmodel,\ell_1)} \routedistance^\tau(\ell_2)$.
\end{definition}
\noindent Note that by the above definition,
$\routedistance^\route(\ell) = 0$ if $\tau[0] = \ell$ and $\infty$ if
$\ell$ is not a part of the route (i.e. $\tau(\ell) = \infty$), and
$\routedistance(\ell_1,\ell_2) = \infty$ if there is no route from
$\ell_1$ to $\ell_2$.

\mypara{Spatio-temporal Time-Series}
A spatio-temporal trace associates each location in a spatial model
with a time-series trace. Formally, a time-series trace $x$ is a
mapping from a time domain $\timedomain$ to some bounded and non-empty
set known as the value domain $\valuedomain$. Given a spatial model
$\spatialmodel = \langle \locations, \wfun \rangle$, a spatio-temporal
trace $\sts$ is a function from $\locations \times \timedomain$ to
$\valuedomain$. We denote the time-series trace at location $\ell$ by
$\sts(\ell)$.

\begin{example}
Consider a spatio-temporal trace $\sts$ of the BSS defined such that
for each location $\ell$ and at any given time $t$, $\sts(\ell,t)$ is
$(\bikes(t),\slots(t))$, where $\bikes(t)$ and $\slots(t)$ are respectively the number of bikes and empty slots at time $t$. 
% We show
% the spatio-temporal trace in Fig.~\ref{fig:bss_traces_runex}.
\end{example}

%%%%%%%%%%%%%%%%%%%%%%%%%%%%%%%%%%%%%%%%%%%
%   STREL
%%%%%%%%%%%%%%%%%%%%%%%%%%%%%%%%%%%%%%%%%%%

\subsection{Spatio-Temporal Reach and Escape Logic
(STREL)}\label{sec:strel}

\mypara{Syntax} STREL is a logic that was introduced in
\cite{bartocci2017monitoring} as a formalism for monitoring spatially
distributed cyber-physical systems. STREL extends Signal Temporal
Logic \cite{maler2004monitoring} with two spatial operators, {\em
reach} and {\em escape}, from which is possible to derive other three
spatial modalities: {\em everywhere}, {\em somewhere} and {\em surround}. The syntax of
STREL is given by:
\[ \varphi ::=  \T  \mid 
    \mu \mid  
    \neg \varphi \mid  
    \varphi_{1} \wedge \varphi_{2} \mid  
    \varphi_{1} \: \until{I} \: \varphi_{2} \mid 
    %\varphi_{1} \: \since{I} \: \varphi_{2} \mid  
    \varphi_{1} \: \reach{\distint} \: \varphi_{2} \mid 
    \escape{\distint}  \: \varphi.  \]
\noindent Here, $\mu$ is an atomic predicate (AP) over the value
domain $\valuedomain$.  Negation $\neg$ and conjunction $\wedge$ are
the standard Boolean connectives, while $\until{I}$ is the temporal
operator {\em until} with $I$ being a non-singular interval over the
time-domain $\timedomain$. The operators $\reach{\distint}$ and
$\escape{\distint}$ are spatial operators where $\distint$ denotes an
interval over the distances induced by the underlying spatial model,
i.e., an interval over $\Reals^{\ge 0}$. 

\mypara{Semantics} A STREL formula is evaluated piecewise over each
location and each time appearing in a given spatio-temporal trace.  We
use the notation $(\sts,\ell) \models \varphi$ if the formula $\varphi$
holds true at location $\ell$ for the given spatio-temporal trace
$\sts$.  The interpretation of atomic predicates, Boolean operations
and temporal operators follows standard semantics for Signal Temporal
Logic: E.g., for a given location $\ell$ and a given time $t$, the
formula $\varphi_1 \until{I} \varphi_2$ holds at $\ell$ iff there is
some time $t'$ in $t \oplus I$ where $\varphi_2$ holds, and for all
times $t''$ in $[t,t')$, $\varphi_1$ holds. Here the $\oplus$ operator
defines the interval obtained by adding $t$ to both interval
end-points. We use standard abbreviations $\ev{I}\varphi =
\mathit{true} \until{I} \varphi$ and $\glob{I}\varphi =
\neg\ev{I}\varphi$, for the {\it eventually} and {\it globally}
operators.  The reachability ($ \reach{\distint}$) and escape
($\escape{\distint}$)operators are spatial operators.  The formula
$\varphi_1 \reach{\distint} \varphi_2$ holds at a location $\ell$ if
there is a route $\route$ starting at $\ell$ that reaches a location
$\ell'$ that satisfies  $\varphi_2$, with a route distance
$\routedistance^\route(\ell')$ that lies in the interval $\distint$,
and for all preceding locations, including $\ell$, $\varphi_1$ holds
true.  The escape formula $\escape{\distint}\varphi$ holds at a
location $\ell$ if there exists a location $\ell'$ at a route distance
$\routedistance(\ell_1,\ell_2)$ that lies in the interval $\distint$
and a route starting at $\ell$ and reaching $\ell'$ consisting of
locations that satisfy $\varphi$.  We define two other operators for
notational convenience: The \textit{somewhere} operator, denoted
$\somewhere{\distint}\varphi$, is defined as $\mathit{true}
\reach{\distint} \varphi$, and the \textit{everywhere} operator,
denoted $\everywhere{\distint}\varphi$ is defined as
$\neg\somewhere{\distint}\neg \varphi$, their meaning is described in
the next example. 

% \textit{everywhere} ($\everywhere{}{}$) are special instances of
% $\reach{}{}$ operator; $\somewhere{d}{f} \varphi := true  \reach{d}{f}
% \varphi$ and $\everywhere{d}{f} \varphi := \neg \somewhere{d}{f} \neg
% \varphi$, and they are defined for notational simplification.
% Furthermore, the \textit{somewhere} and \textit{everywhere} operators
% are more easier to interpret. $\somewhere{d}{f} \varphi$ means that at
% least for one of the reachable locations from the initial location the
% property $\varphi$ holds, such that the distance from the initial
% location and the final one belongs to the interval $d$. The property
% $\everywhere{d}{f} \varphi$ means that for all locations reachable
% from the initial location that their distance to initial location
% belongs to the interval $d$, the property $\varphi$ is satisfied. 

\begin{example} 
\label{ex:bss_strel}
In the BSS, we use atomic predicates $\slots > 0$ and $\bikes > 10$,
and the formula
$\glob{[0,3\mathrm{hours}]}\somewhere{[0,1\mathrm{km}]}(\bikes > 10)$ is
true if always within the next $3$ hours, at a location $\ell$, there is some location $\ell'$ at most
$1$ km from $\ell$ where, the number of bikes
available exceed $10$.  Similarly, the formula
$\everywhere{[0,1\mathrm{km}]}\glob{[0,30\mathrm{min}]}(\slots > 0)$
is true at a location $\ell$ if for all locations within $1$km, for
the next $30$ mins, there is no empty slot.  
\end{example}

% \textcolor{green}{G has the temporal constraint not as subscript, sometimes we use B,S and sometimes b,s}

% STREL is equipped with both Boolean and quantitative semantics; a
% Boolean semantics, $(\mathcal{S}, x, \ell, t) \models \varphi$, with
% the meaning that the spatio-temporal trace $x$ in location $\ell$ at
% time $t$ with spatial model $\mathcal{S}$, satisfies the formula
% $\varphi$ and a quantitative semantics, $\rho(\varphi, \mathcal{S}, x,
% t)$, that can be used to measure the quantitative level of
% satisfaction of a formula for a given trajectory and space model. The
% function $\rho$ is also called the robustness function. The
% satisfaction of the whole trajectory corresponds to the safisfaction
% at time $0$, i.e. $\rho(\varphi, \mathcal{S}, x) = \rho(\varphi,
% \mathcal{S}, x, 0)$. Detailed explanation about Boolean and
% quantitative semantics of STREL is provided in Appendix. In our
% current work, we use an existing monitoring tool MoonLight
% \cite{bartocci2020moonlight} for computing the quantitative semantics
% of STREL formulas.

%$=\min_{(\ell, w, \ell') \in \wfun}(f(w))$, where $f$ is the distance function, 

\section{Constructing a Spatial Model}
\label{sec:graph}
\begin{figure}[t]
\centering
% \setkeys{Gin}{height=2cm} 
% \makebox[\textwidth]{%
%   \setlength{\tabcolsep}{3pt}%
%  \begin{tabular}{@{}cc@{}}
%
\subfloat[\label{fig:full_graph}]{%
\includegraphics[width=.249\textwidth]{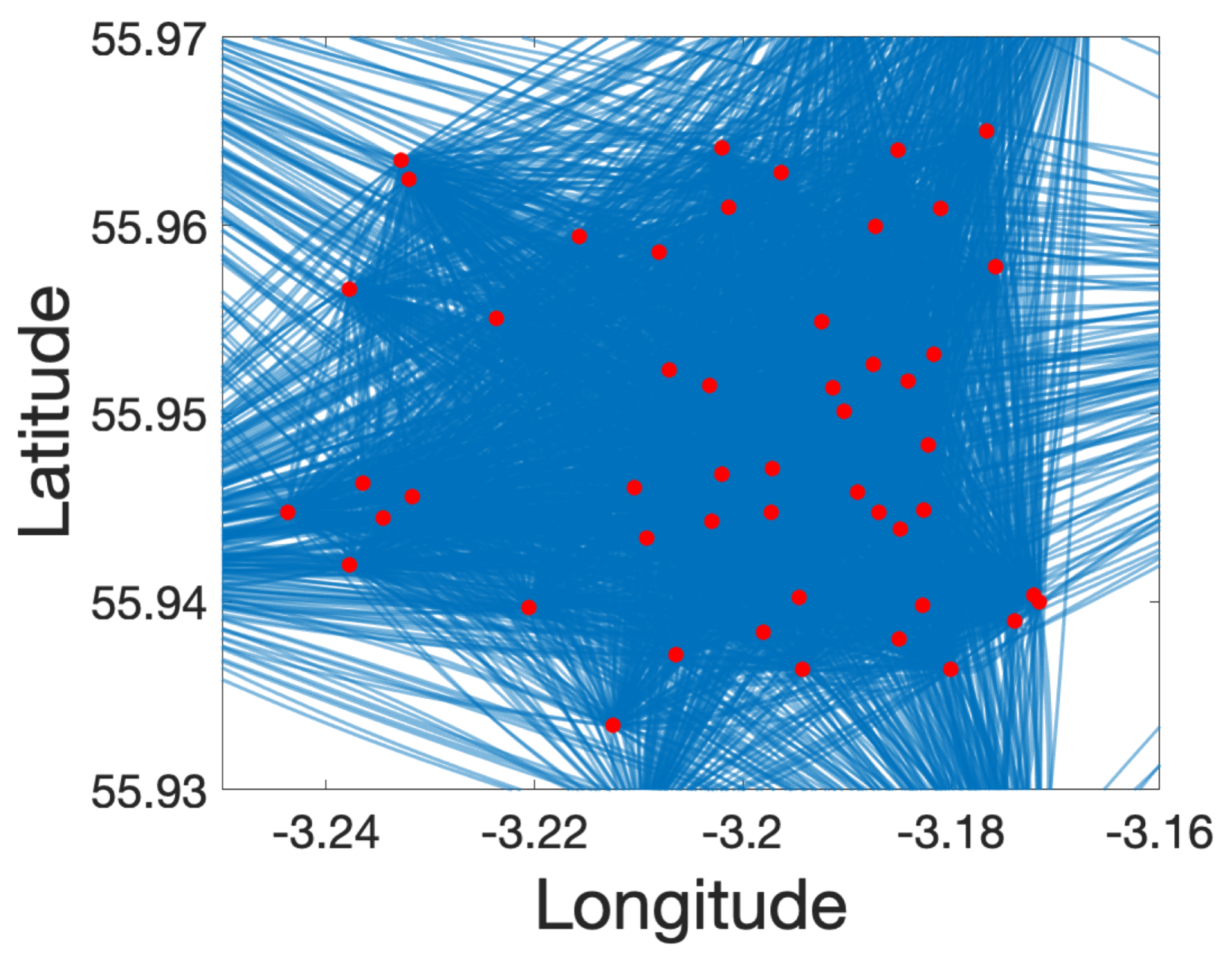} %
} %
\subfloat[\label{fig:delta_graph}]{%
\includegraphics[width=.249\textwidth]{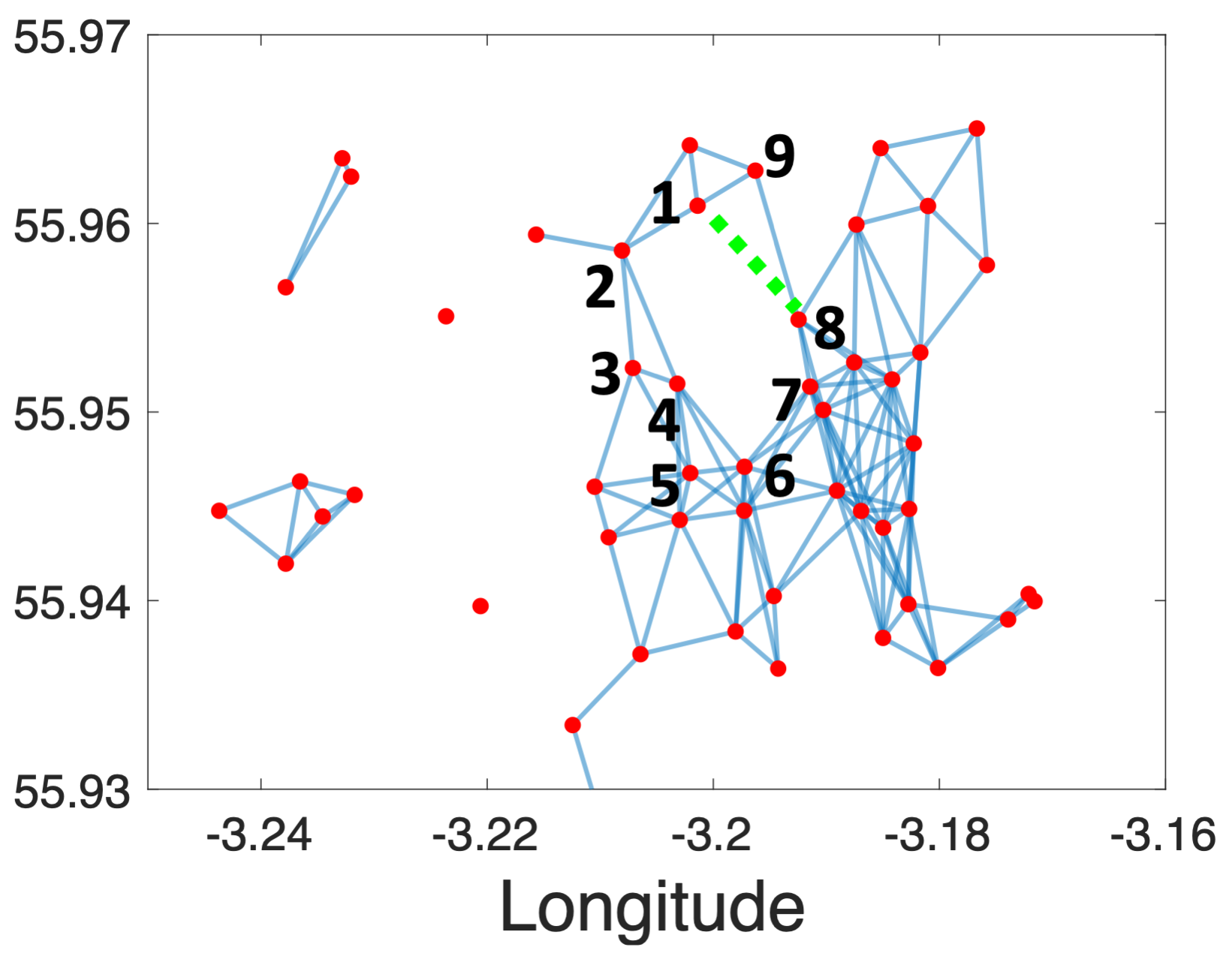} %
} %
\subfloat[\label{fig:mst_gprah}]{%
\includegraphics[width=.249\textwidth]{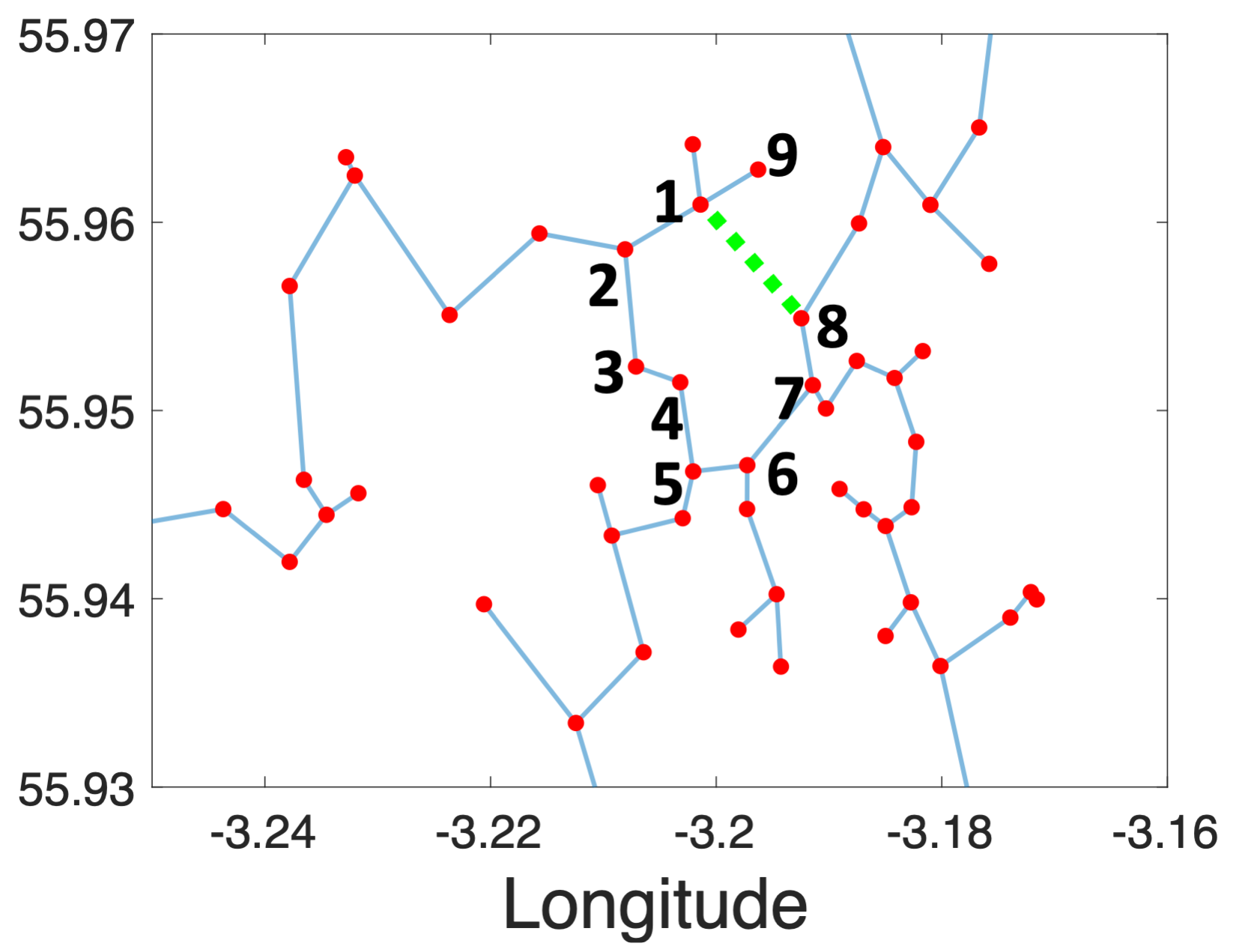}%
} %
\subfloat[\label{fig:enhanced_mst}]{%
\includegraphics[width=.249\textwidth]{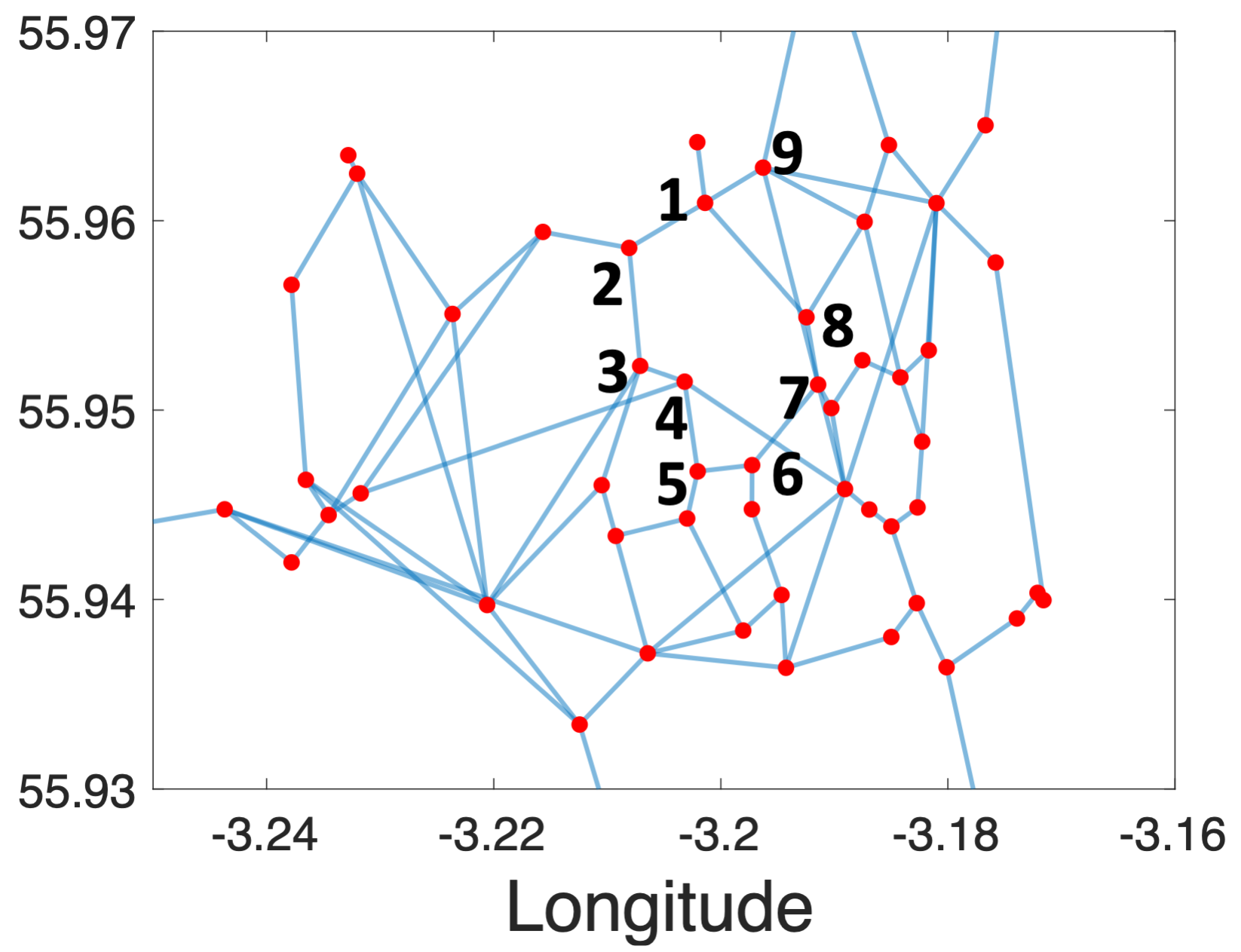}%
}
\caption{Different approaches for constructing the spatial model for
the BSS.  (a) shows an $(\infty,\dhaver)$-connectivity spatial model
where $\dhaver$ is the Haversine distance between locations. (b) shows
a $(\delta,\dhaver)$-connectivity spatial model where $\delta=1km$.
Observe that the spatial model is disconnected. (c) shows an
MST-spatial model. (d) shows an $(\alpha,\dhaver)$ enhanced MSG
spatial model with $\alpha=2$. Observe that this spatial model is
sparse compared even to the $(\delta,\dhaver)$-connectivity spatial
model.}
\label{fig:graph}
\end{figure}

% given a set of locations $L$ over some metric space. In this paper, we
% assume that the locations are places on a map and thus defined by
% their latitude and longitude, so we use the Haversine formula
% \cite{chopde2013landmark} as the metric. The haversine distance is the
% great-circle distance between two points on a sphere given their
% longitudes and latitudes.

In this section, we present four approaches to construct a
spatial model, and discuss the pros and cons of each approach.

\begin{enumerate}[align=left, leftmargin=0pt, labelindent=0em, listparindent=\parindent, labelwidth=0pt, itemindent=!]

\item {\bf $(\infty,d)$-connectivity spatial model}: This spatial model
corresponds to the $(\delta,d)$-connectivity spatial model as
presented in Definition~\ref{def:deltaconngraph}, where we set $\delta =
\infty$. We note that this gives us a fully connected graph, i.e.
where $|\wfun|$ is $O(|\locations|^2)$. We remark that our learning
algorithm uses monitoring STREL formulas as a sub-routine, and from
Lemma~\ref{lem:monitoring_complexity} in Appendix, we can see that as the
complexity of monitoring a STREL formula is linear in $|\wfun|$, a
fully connected graph is undesirable.

\item {\bf $(\delta,d)$-connectivity spatial model}: This is the model
presented in Definition~\ref{def:deltaconngraph}, where $\delta$ is
heuristically chosen in an application-dependent fashion. Typically,
the $\delta$ we choose is much smaller compared to the distance
between the furthest nodes in the given topological space.  This gives
us $\wfun$ that is sparse, and thus with a lower monitoring cost;
however, a small $\delta$ can lead to a disconnected spatial model
which can affect the accuracy of the learned STREL formulas.
Furthermore, this approach may overestimate the spatial model induced
distance between two nodes (as in Definition~\ref{eq:routdist}) that
are not connected by a direct edge.  For instance, in
Fig.~\ref{fig:delta_graph}, nodes $1$ and $8$ are connected through
the route $1 \rightarrow 9 \rightarrow 8$, and sum of
the edge-weights along this route is larger than the actual (metric)
distance of $1$ and $8$. 

\item {\bf MST-spatial model}: To minimize the number of edges in the
graph while keeping the connectivity of the graph, we can use
Minimum Spanning Tree (MST) as illustrated in
Fig.~\ref{fig:mst_gprah}. This gives us $|\wfun|$ that is
$O(|\locations|)$, which makes monitoring much faster, while resolving
the issue of disconnected nodes in the $(\delta,d)$-spatial model.
However, an MST can also lead to an overestimate of the spatial model
induced distance between some nodes in the graph. For example, in
Fig.~\ref{fig:mst_gprah}, the direct distance between nodes $1$ and
$8$ is much smaller than their route distance (through the route $1
\rightarrow 2 \rightarrow 3 \rightarrow 4 \rightarrow 5 \rightarrow
6 \rightarrow 7 \rightarrow
8$). 

\item {\bf $(\alpha,d)$-Enhanced MSG Spatial Model}: To address the
shortcomings of previous approaches, we propose constructing a spatial
model that we call the $(\alpha,d)$-{\em Enhanced Minimum Spanning
Graph} Spatial model.  First, we construct an MST over the given set
of locations and use it to define $\wfun$ and pick $\alpha$ as some
number greater than $1$.  Then, for each distinct pair of locations
$\ell_1,\ell_2$, we compute the shortest route distance
$\routedistance(\ell_1,\ell_2)$ between them in the constructed MST,
and compare it to their distance $d(\ell_1,\ell_2)$ in the metric
space. If $\routedistance(\ell_1,\ell_2) > \alpha \cdot
d(\ell_1,\ell_2)$, then we add an edge $(\ell_1, d(\ell_1,\ell_2),
\ell_2)$ to $\wfun$.  The resulting spatial model is no longer a tree,
but typically is still sparse. The
complete algorithm, is provided in Algo.~\ref{alg:graph}, which is a simple way of constructing an
$(\alpha,d)$-enhanced MSG spatial model, and incurs a one-time cost of
$O(|\locations|^2 \cdot (|\locations|+|\wfun| \cdot
\log(|\locations|)))$. We believe that the time complexity can be
further improved using a suitable dynamic programming based
approach.  In our case studies, the cost of building the enhanced
MSG spatial model was insignificant compared to the other steps in the
learning procedure. The runtimes of our learning approach for
different kinds of spatial models on various case studies is
illustrated in Table.~\ref{tab:graph}. We also represent information about the number of isolated nodes in each graph. We consider the run-time greater than 30 minutes as time-out. The results demonstrate that  $(\infty,\dhaver)$-connectivity spatial model usually results in time-out because of the large number of edges. While $(\delta,\dhaver)$-connectivity spatial model results in a better run-time, it has the problem of isolated nodes or dis-connectivity. MST-spatial model neither results in time-out or dis-connectivity; however, it has the problem of overestimating distances between nodes. While our approach, $(\alpha,\dhaver)$-Enhanced MSG Spatial Model, has a slightly worse run-time compared to MST-spatial model, it improves the distance over-approximation issue.
\end{enumerate}

\begin{table}[ht]
\centering
\begin{tabular*}{.95\textwidth}{@{\extracolsep{\fill}}lllll}
\toprule
%\hline
Case &  $(\infty,\dhaver)-$   & $(\delta,\dhaver)-$ & MST & $(\alpha,\dhaver)-$\\
&connectivity&connectivity & &Enhanced MSG \\
%\hline
\midrule
COVID-19&time-out,0&1007.44s,75&600.89s,0&813.65s,0 \\
BSS&time-out,0&934.41s,17&519.30s,0&681.78s,0 \\
Air Quality&time-out,0&111.36s,46&119.94s,0&136.02s,0 \\
Food Court &170.62s,0&84.25s,0&73.53s,0&78.24s,0  \\ [1ex]
\bottomrule
\end{tabular*}
\caption{Run time of the learning algorithm (seconds), number of isolated nodes in the spatial model . , . (threshold for time-out is set to 30 minutes).\label{tab:graph}}
\end{table}

\begin{algorithm}[t]
\DontPrintSemicolon
\KwIn{A set of locations $\locations$  (vertices of the graph), Longitudes and Latitudes of the locations, factor $\alpha > 1$}
\KwOut{$\spatialmodel$} 
    $\spatialmodel = minSpanningTree(\locations, \mathsf{Longitudes}, \mathsf{Latitudes})${\color{teal}\tcp{Prim's algorithm}} \;
    \For{$i \leftarrow 1$ to $|\locations|$}{
        \For{$j \leftarrow i+1$ to $|\locations|$}{
            {\color{teal}\tcp{Length of the shortest path between i and j}}
            $shortestPathMST_{ij} = length(shortestPath(\spatialmodel,i,j))$\\
            {\color{teal}\tcp{Compute the Haversine distance between i and j}}
            $\mathsf{directDistance}_{ij} = \dhaver(\mathsf{longitudes}[i,j], \mathsf{latitudes}[i,j])$\\
            \lIf{$shortestPathMST_{ij} > \alpha \cdot directDistance_{ij}$}{
            $addEdge(\spatialmodel, i, j)$
        } 
    }
}
\Return $\spatialmodel$
\caption{Algorithm to create an $(\alpha,\dhaver)$-Enhanced MSG Spatial Model \label{alg:graph}}
\end{algorithm}

\section{Learning STREL formulas from data}
In this section, we first introduce Parametric Spatio-Temporal Reach
and Escape Logic (PSTREL) and the notion of monotonicity for PSTREL
formulas. Then, we introduce a projection function $\pi$ that maps a
spatio-temporal trace to a valuation in the parameter space of a given
PSTREL formula. We then cluster the trace-projections using
Agglomerative Hierarchical Clustering, and finally learn a compact
STREL formula for each cluster using Decision Tree techniques.

\mypara{Parametric STREL (PSTREL)}
% \label{subsec:pstrel}
Parametric STREL (PSTREL) is a logic obtained by replacing one or more
numeric constants appearing in STREL formulas by parameters;
parameters appearing in atomic predicates are called {\em magnitude}
parameters $\valueparams$, and those appearing in temporal and spatial
operators are called {\em timing} $\timeparams$ and {\em spatial}
parameters $\spaceparams$ respectively. Each parameter in
$\valueparams$ take values from $\valuedomain$, those in
$\timeparams$ take values from $\timedomain$, and those in
$\spaceparams$ take values from $\Reals^{\ge 0}$ (i.e. the set of
values that the $\routedistance$ metric can take for a given spatial
model). We define a valuation function $\val$ that maps all parameters
in a PSTREL formula to their respective values. 

\begin{example}
Consider the PSTREL versions of the STREL formulas introduced in
Example~\ref{ex:bss_strel} $\varphi(\param_\tau,\param_d,\param_c)$ = 
$\mathbf{G}_{[0,\param_\tau]}\somewhere{[0,\param_d]}(\bikes > \param_c)$.
The valuation $\val$: $\param_\tau
\mapsto 3\mathrm{hours}$, $\param_d \mapsto 1\mathrm{km}$, and $\param_c \mapsto 10$ returns the STREL
formula introduced in Example~\ref{ex:bss_strel}. 
\end{example}

\begin{definition}[Parameter Polarity, Monotonic PSTREL]
A polarity function $\polarity$ maps a parameter to an element of $\{+,-\}$, and
is defined as follows:
\vspace{1mm}

{\centering
  $ \displaystyle
    \begin{aligned} 
\polarity(\param)& = +\ \ \eqdef\ \ 
\val'(\param) > \val(\param) \wedge (\sts,\ell) \models \varphi(\val(\param))  
\Rightarrow (\sts,\ell) \models \varphi(\val'(\param)) \\
\polarity(\param)& = -\ \ \eqdef\ \ 
\val'(\param) < \val(\param) \wedge (\sts,\ell) \models \varphi(\val(\param))  
\Rightarrow (\sts,\ell) \models \varphi(\val'(\param)) \\
    \end{aligned}
  $ 
\par}
\vspace{2mm}
The monotonic fragment of PSTREL consists of PSTREL formulas where all
parameters have either positive or negative polarity.
\end{definition}
In simple terms, the polarity of a parameter $\param$ is positive if
it is easier to satisfy $\varphi$ as we increase the value of $\param$
and is negative if it is easier to satisfy $\varphi$ as we decrease
the value of $\param$.  The notion of polarity for PSTL formulas was
introduced in \cite{asarin2011parametric}, and we extend this to
PSTREL and spatial operators.  The polarity for PSTREL formulas
$\varphi(d_1,d_2)$ of the form $\somewhere{[d_1,d_2]}\psi$,
$\psi_1\reach{[d_1,d_2]}\psi_2$, and $\escape{[d_1,d_2]}\psi$ are
$\polarity(d_1) = -$ and $\polarity(d_2) = +$, i.e. if a
spatio-temporal trace satisfies $\varphi(\val(d_1),\val(d_2))$, then
it also satisfies any STREL formula over a strictly larger spatial
model induced distance interval, i.e. by decreasing $\val(d_1)$ and
increasing $\val(d_2)$.  For a formula $\everywhere{[d_1,d_2]}\psi$,
$\polarity(d_1) = +$ and $\polarity(d_2) = -$, i.e. the formula
obtained by strictly shrinking the distance interval. The proofs are
simple, and provided in Appendix for completeness.
% \begin{wraptable}{r}{.4\textwidth}
% \centering
% \begin{tabular*}{.6\textwidth}{@{\extracolsep{\fill}}lll}
% %\toprule
% \hline
% Spatial operator & $\pi(d_1)$   & $\pi(d_2)$\\
% \hline
% %\midrule
% $\reach{[d_1, d_2]}{f}$&$-$&$+$ \\
% $\escape{[d_1, d_2]}{f}$&$-$&$+$ \\
% $\somewhere{[d_1, d_2]}{f}$&$-$&$+$ \\
% $\everywhere{[d_1, d_2]}{f}$&$+$&$-$  \\
% % $\surround{[d_1, d_2]}{f}$&$-$&$+$  \\ [1ex]
% %\bottomrule
% \hline
% \end{tabular*}
% \caption{Polarity of parameters for spatial operators. For each of the
% formulas, $\pi(d_1)$ is defined when $d_2$ is a fixed constant, and
% $\pi(d_2)$ is defined 
% 
% \label{tab:polarity}}
% \end{wraptable}

\begin{definition}[Validity Domain, Boundary]
Let $\paramspace = \valuedomain^{|\valueparams|} \times \timedomain^{|
\timeparams|}\times (\Reals^{\ge 0})^{|\spaceparams|}$ denote the
space of parameter valuations, then the validity domain
$\validitydomain$ of a PSTREL formula at a location $\ell$ with
respect to a set of spatio-temporal traces $\traces$ is defined as
follows: 
% \begin{align*}
\(
\validitydomain(\varphi(\param),\ell,\traces)  =
\{ \val(\param) \mid \param \in \paramspace, \sts \in \traces,
(\sts,\ell) \models \varphi(\val(\param)) \} \)
% \end{align*}
The validity domain boundary
$\validitydomainboundary(\varphi(\varphi),\ell,\traces)$ is defined as the
intersection of $\validitydomain(\varphi,\ell,\traces)$ with the closure of
its complement.
\end{definition}

\mypara{Spatio-Temporal Trace Projection}
% \label{subsec:trace_projection}
We now explain how a monotonic PSTREL formula $\varphi(\param)$ can be
used to automatically extract features from a spatio-temporal trace.
The main idea is to define a total order $>_\params$ on the parameters
$\param$ (i.e. parameter priorities) that allows us to define a
lexicographic projection of the spatio-temporal trace $\sts$ at each
location $\ell$ to a parameter valuation $\val(\param)$ (this is
similar to assumptions made in
\cite{jin2015mining,vazquez2017logical}). We briefly remark how we can
relax this assumption later. Let $\val_j$ denote the valuation of the
$j^{th}$ parameter.  \begin{definition}[Parameter Space Ordering,
Projection] A total order on parameter indices $j_1 > \ldots > j_n$
imposes a total order $\preclex$ on the parameter space defined as:
\[
\val(\param) \preclex \val'(\param) \Leftrightarrow \exists j_k 
\text{ s.t. }\left\{\begin{array}{l} 
\polarity(\param_{j_k}) = + \Rightarrow \val_{j_k} < \val'_{j_k} \\
\polarity(\param_{j_k}) = - \Rightarrow \val_{j_k} > \val'_{j_k} 
\end{array}\right.\, \mathrm{and}\ \forall m <_\params k, \val_m =
\val'_m.
\]
Given above total order, $\projlex(\sts,\ell) = \inf_{\preclex}
\{\val(\param) \in \validitydomainboundary(\varphi(\param),\{\sts\}\}$.
\end{definition}
In simple terms, given a total order on the parameters, the
lexicographic projection maps a spatio-temporal trace to valuations
that are least permissive w.r.t. the parameter with the greatest
priority, then among those valuations, to those that are least
permissive w.r.t. the parameter with the next greater priority, and so
on. Finding a lexicographic projection can be done by sequentially
performing binary search on each parameter dimension
\cite{vazquez2017logical}. It is easy to show that $\projlex$ returns
a valuation on the validity domain boundary. The method for finding the lexicographic projection is formalized in
Algo.~\ref{alg:bisection_search}. The algorithm begins by setting the
lower and upper bounds of valuations for each parameter. Then, each
parameter is set to a parameter valuation that results in the most
permissive STREL formula (based on the monotonicity direction). Next,
for each parameter in the defined order $>_\params$ we perform bisection
search to learn a tight satisfying parameter valuation.  After
completion of bisection search, we return the upper (lower) bound of
the search interval for parameters with positive (negative) polarity.

\begin{algorithm}[t]
\KwIn{A trace $\sigma(\ell)$, a spatial model $\mathcal{S}$, a PSTREL formula $\varphi(\param)$, a parameter set $\mathcal{P}$, monotonicity directions $\gamma(\param)$, defined order on parameters $>_\params$, $\mathbf{\delta} > 0$}
{\color{teal}\tcp{$\projlex(\sts,\ell)$ is the projection of $\sts(\ell)$ to a point in the parameter space of $\varphi$}}
\KwOut{$\projlex(\sts,\ell)$} 

{\color{teal}\tcp{Lower and upper bounds of each parameter}}
$\val^l(\mathbf{p}) \leftarrow \mathbf{inf} (\mathcal{P})$,  $\val^u(\mathbf{p}) \leftarrow \mathbf{sup} (\mathcal{P})$ \;

{\color{teal}\tcp{Initialize each parameter with a value in the parameter space that results in the most permissive formula (based on the monotonicity direction of each parameter)}}
\For{$i \leftarrow 1$ to $|\mathcal{P}|$}{
\lIf{$\gamma(\param_i) == +$}{$\val(\param_i) \leftarrow \val^u(\param_i)$}
\lElse{$\val(\param_i) \leftarrow \val^l(\param_i)$}
}
{\color{teal}\tcp{Optimize each parameter in the defined order $>_\params$}}

\For{$i \leftarrow 1$ to $|\params_>|$}{
    \While{$|\val^u(\param_i) - \val^l(\param_i)| \geq \delta_i$}{
        $\val(\param_i)= \frac{1}{2}(\val^l(\param_i) + \val^u(\param_i))$ {\color{teal}\tcp{The middle point}} \;
        {\color{teal}\tcp{Compute robustness of the middle point}}
        $\rho = \rho(\varphi(\val(\param_i)), \mathcal{S}, \sigma, \ell, 0)$;\\
        \lIf{$\rho \geq 0 \;\&\; \gamma(\param_i) == +$}{
            $\val^u(\param_i) \leftarrow \val(\param_i)$
        }
        \lElseIf{$\rho \geq 0 \;\&\; \gamma(\param_i) == -$}{
            $\val^l(\param_i) \leftarrow \val(\param_i)$
        }
        \lElseIf{$\rho < 0 \;\&\; \gamma(\param_i) == +$}{
            $\val^l(\param_i) \leftarrow \val(\param_i)$
        }
        \lElse{
            $\val^u(\param_i) \leftarrow \val(\param_i)$
        }
    }
    \lIf{$\gamma(p_i) == +$}{$\val_{final}(\param_i) \leftarrow v^u(\param_i)$}
    \lElse{$\val_{final}(\param_i) \leftarrow \val^l(\param_i)$ \;}
}
\Return $\projlex(\sts,\ell) \leftarrow \val_{final}(\param)$
\caption{Lexicographic projection of spatio-temporal traces using multi-dimensional bisection search \label{alg:bisection_search}}
\end{algorithm}

\begin{remark}
The order of parameters is assumed to be provided by the user and is
important as it affects the unsupervised learning algorithms for
clustering that we apply next. Intuitively, the order corresponds to
what the user deems as more important. For example, consider the
formula $\mathbf{G}_{[0,3\mathrm{hours}]}\somewhere{[0,d]} (\bikes > c)$. Note
that $\polarity(d) = +$, and $\polarity(c) = -$. Now if the user is
more interested in the radius around each station where the number of
bikes exceeds some threshold (possibly $0$) within $3$hours, then the
order is $d >_\params c$. If she is more interested in knowing what is
the largest number of bikes available in any radius (possibly
$\infty$) always within $3$hours, then $c >_\params d$. 
\end{remark}

\begin{remark} Similar to \cite{vazquez2018time}, we can compute an
approximation of the validity domain boundary for a given trace, and then apply a clustering algorithm on the
validity domain boundaries.  This does not require the user to specify
parameter priorities. In all our case studies, the parameter
priorities were clear from the domain knowledge, and hence we will
investigate this extension in the future.
\end{remark}

% \begin{example}
% For the PSTREL formula  $\varphi(\param_d,\param_\tau,\param_c)$ = 
% $\somewhere{[0,\param_d]}\ev{[0,\param_\tau]}(\bikes > \param_c)$, we assume that the parameter space of $\param_d \in [10, 1100]$, $\param_\tau \in [100, 150]$ and  $\param_c \in [0, 100]$ and the ordering of $\param_c \preceq \param_\tau \preceq \param_d$ or $(\param_c, \param_\tau, \param_d)$ is defined by the user. The bisection search method first initializes each parameter with a value in the parameter space that results in the most permissive formula. Since $\pi(\param_c) = -$, $\pi(\param_\tau) = +$ and $\pi(\param_d) = +$, the parameter initialization is $\param_c = 0,  \param_\tau = 150 , \param_d=1100$. Next, it tries to find a tight parameter valuations that make the formula marginally satisfied by the given trace. The defined ordering has the effect of first searching for the largest $\param_c$, then, searching for the smallest $\param_\tau$, and finally searching for the smallest $\param_d$.
% \end{example}
% 

\mypara{Clustering}
% \label{susec:clustering}
The projection operator $\projlex(\sts,\ell)$ maps each location to a
valuation in the parameter space. These valuation points serve as
features for off-the-shelf clustering algorithms. In our experiments,
we use the {\em Agglomerative Hierarchical Clustering} (AHC) technique
\cite{day1984efficient}
to automatically cluster similar valuations.  AHC is a bottom-up
approach that starts by assigning each point to a single cluster, and
then merging clusters in a hierarchical manner based on a similarity
criteria\footnote{We used complete-linkage criteria which assumes the distance between clusters equals the distance between those two elements (one in each cluster) that are farthest away from each other.}. An
important hyperparameter for any clustering algorithm is the number of
clusters to choose. In some case studies, we use domain knowledge to
decide the number of clusters. Where such knowledge is not available,
we use the {\em Silhouette metric} to compute the optimal number of
clusters. Silhouette is a ML method to interpret and
validate consistency within clusters by measuring how well each point
has been clustered. The silhouette metric ranges from $-1$ to $+1$,
where a high silhouette value indicates that the object is well
matched to its own cluster and poorly matched to neighboring clusters
\cite{rousseeuw1987silhouettes}. 

\begin{example}
Fig.~\ref{fig:bss_clustering_runex} shows the results of projecting
the spatio-temporal traces from BSS through the PSTREL formula $\varphi(\tau,d)$
shown in Eq.~\eqref{eq:bss_formula}.
\begin{equation}
\varphi(\tau, d) = \glob{[0,3]}(\fwait(\tau) \vee \fwalk(d)) \label{eq:bss_formula} 
\end{equation}
In the above formula, $\fwait(\tau)$ is defined as
$\ev{[0,\tau]}(\bikes \geq 1) \wedge (\ev{[0,\tau]} \slots \geq 1)$,
and $\fwalk(d)$ is $\somewhere{[0,d]} (\bikes \geq 1) \wedge
\somewhere(\slots \geq 1)$.  $\varphi(\tau,d)$ means that for the next
3 hours, either $\fwait(\tau)$ or $\fwalk(d)$ is true. Locations with
large values of $\tau$ have long wait times or with large $d$ values
are typically far from a location with bike/slot availability (and are thus
undesirable).  Locations with small $\tau,d$ are desirable.  Each
point in Fig.~\ref{fig:bss_clustering_runex} shows
$\projlex(\sts,\ell)$ applied to each location and the result of
applying AHC with $3$ clusters.
\end{example}

Let $\numclusters$ be the number of clusters obtained after applying
AHC to the parameter valuations.  Let $\clusterlabel$ denote the
labeling function mapping $\projlex(\sts,\ell)$ to
$\{1,\ldots,\numclusters\}$.  The next step after clustering is to
represent each cluster in terms of an easily interpretable STREL
formula. Next, we propose a decision tree-based approach to learn an
interpretable STREL formula from each cluster.

\mypara{Learning STREL Formulas from Clusters}
\label{subsec:learn_strel}
The main goal of this subsection is to obtain a compact STREL formula
to describe each cluster identified by AHC. We argue that bounded
length formulas tend to be human-interpretable, and show how we can
automatically obtain such formulas using a decision-tree approach.
Decision-trees (DTs) are a non-parametric supervised learning method
used for classification and regression\cite{Mitchell:1997:ML:541177}.
Given a finite set of points $X \subseteq \Reals^m$ and a labeling
function $\lfun$ that maps each point $x \in X$ to some label
$\lfun(x)$, the DT learning algorithm creates a tree whose non-leaf
nodes $n_j$ are annotated with constraints $\phi_j$, and each leaf
node is associated with some label in the range of $\lfun$.  Each path
$n_1,\ldots,n_i,n_{i+1}$ from the root node to a leaf node corresponds
to a conjunction $\bigwedge_{j=1}^{i} h_j$, where $h_j$ = $\neg
\phi_j$ if $h_{j+1}$ is the left child of $h_j$ and $\phi_j$
otherwise. Each label thus corresponds to the disjunction over the
conjunctions corresponding to each path from the root node to the leaf
node with that label. 

Recall that after applying the AHC procedure, we get one valuation
$\projlex(\sts,\ell)$ for each location, and its associated cluster
label. We apply a DT learning algorithm to each point
$\projlex(\sts,\ell)$, and each DT node is associated with a $\phi_j$
of the form $p_j \ge v_j$ for some $p_j \in \param$. 

\begin{lemma}
Any path in the DT corresponds to a STREL formula of length that is
$O((|\params| + 1)\cdot|\varphi|)$.
\end{lemma}
\begin{proof}

Any path in the DT is a conjunction over a number of formulas of the
kind $p_j \ge v_j$ or its negation. Because $\varphi(\param)$ is
monotonic in each of its parameters, if we are given a conjunction of
two conjuncts of the type $p_j \ge v_j$ and $p_j \ge v'_j$, then
depending on $\polarity(p_j)$, one inequality implies the other, and
we can discard the weaker inequality. Repeating this procedure, for
each parameter, we will be left with at most $2$ inequalities (one
specifying a lower limit and the other an upper limit on $p_j$). Thus,
each path in the DT corresponds to an axis-aligned hyperbox in the
parameter space. Due to monotonicity, an axis-aligned hyperbox in the
parameter space can be represented by a formula that is a conjunction
of $|\params|+1$ STREL formulas (negations of formulas corresponding
to the $|\params|$ vertices connected to the vertex with the most
permissive STREL formula, and the most permissive formula itself)
\cite{vazquez2017logical} (see Fig.~\ref{fig:hyper_box} for an example
in a 2D parameter space). Thus, each path in the DT can be described
by a formula of length $O((|\params|+1)\cdot|\varphi|)$, where
$|\varphi|$ is the length of $\varphi$.
\end{proof}

\begin{example}
The result of applying the DT algorithm to the clusters identified by
AHC (shown in dotted lines in Fig.~\ref{fig:bss_clustering_runex}) is
shown as the axis-aligned hyperboxes. Using the meaning of
$\varphi(\tau,d)$ as defined in Eq.~\eqref{eq:bss_formula}, we learn
the formula $\neg\varphi(17.09,2100) \wedge \neg \varphi(50,1000.98)
\wedge \varphi(50,2100)$ for the red cluster. The last of these conjuncts is essentially
the formula $\true$, as this formula corresponds to the most
permissive formula over the given parameter space. Thus, the formula
we learn is:
\[ \varphi_{red} = \neg \glob{[0,3]}(\fwait(17.09) \vee \fwalk(2100))
            \wedge \neg \glob{[0,3]}(\fwait(50) \vee \fwalk(1000.98))
\]
The first of these conjuncts is associated with a short wait time and
the second is associated with short walking distance. As both are not
satisfied, these locations are the least desirable.
\end{example}

% Here, we try to learn an easily interpretable STREL formula for each
% cluster of nodes. We prove that using a Decision Tree based approach
% we can learn STREL formulas with bounded description complexity.

% \scalebox{0.6}{\parbox{\linewidth}{
% \begin{align*}
% &\varphi_{red} = \varphi(4220.50,747.49) \wedge \neg \varphi(7000,747.49) \wedge \neg \varphi(4220.50,0) =\varphi(4220.50,747.49) \wedge \neg (False) \wedge \neg (False)= \varphi(4220.50,747.49) \\ \\&\varphi_{blue} = \varphi(1711.50,4000) \wedge \neg (False) \wedge \neg \varphi(1711.50,747.49) = \varphi(1711.50,4000) \wedge \neg \varphi(1711.50,747.49) \\ \\ &\varphi^1_{green} = \varphi(0,4000) \wedge \neg \varphi(1711.50,4000) \wedge \neg \varphi(0,0) =True \wedge \neg \varphi(1711.50,4000) \wedge \neg (False)= \neg \varphi(1711.50,4000)\\ &\varphi^2_{green} = \varphi(1711.50,747.49) \wedge \neg \varphi(4220.50,747.49) \wedge \neg \varphi(1711.50,0) = \varphi(1711.50,747.49) \wedge \neg \varphi(4220.50,747.49) \wedge \neg (False) \\ &= \varphi(1711.50,747.49) \wedge \neg \varphi(4220.50,747.49) \\ &\varphi_{green} = \varphi^1_{green} \vee \varphi^2_{green}.
% \end{align*}
% }}

% \begin{proof}
% Each path in the decision tree can be simplified to at most $2n$ conditions using the monotonicity of PSTREL. The $2n$ conditions on parameters represents a hyper-box, and each hyperbox with $n$ parameters can be described using a STREL formula with length at most $(n+1) \cdot |\varphi(p)|$. \textcolor{green}{not clear for me sorry :(}
% \end{proof}

\mypara{Pruning the Decision Tree} If the decision tree algorithm
produces several disjuncts for a given label (e.g., see
Fig.~\ref{fig:covid_exampl1_clustering}), then it can significantly
increase the length and complexity of the formula that we learn for a
label. This typically happens when the clusters produced by AHC are
not clearly separable using axis-aligned hyperplanes. We can mitigate
this by pruning the decision tree to a maximum depth, and in the
process losing the bijective mapping between cluster labels and small
STREL formulas. We can still recover an STREL formula that is
satisfied by most points in a cluster using a $k$-fold cross
validation approach (see Algo.~\ref{alg:dt_prune}). The idea is to loop over
the maximum depth permitted from $1$ to $N$, where $N$ is user
provided, and for each depth performing $k$-fold cross validation to
characterize the accuracy of classification at that depth.  If the
accuracy is greater than a threshold ($90\%$ in our experiments), we
stop and return the depth as a limit for the decision tree.
Fig.~\ref{fig:covid_exampl1_pruned} illustrates the hyper-boxes
obtained using this approach. For this example, we could decrease the
number of hyper-boxes from 11 to 3 by miss-classifying only a few data
points (less than $10\%$ of the data). 

\begin{figure}[t]
\centering
% \setkeys{Gin}{height=3cm} % just for the example
% \makebox[\textwidth]{%
%   \setlength{\tabcolsep}{3pt}%
%   \begin{tabular}{@{}cc@{}}
\subfloat[$\varphi_{yellow} = \varphi(\tau_2,d_2) \wedge \neg \varphi_{pink} \wedge \neg \varphi_{blue}$.\label{fig:hyper_box}]{\includegraphics[width = .4\textwidth]{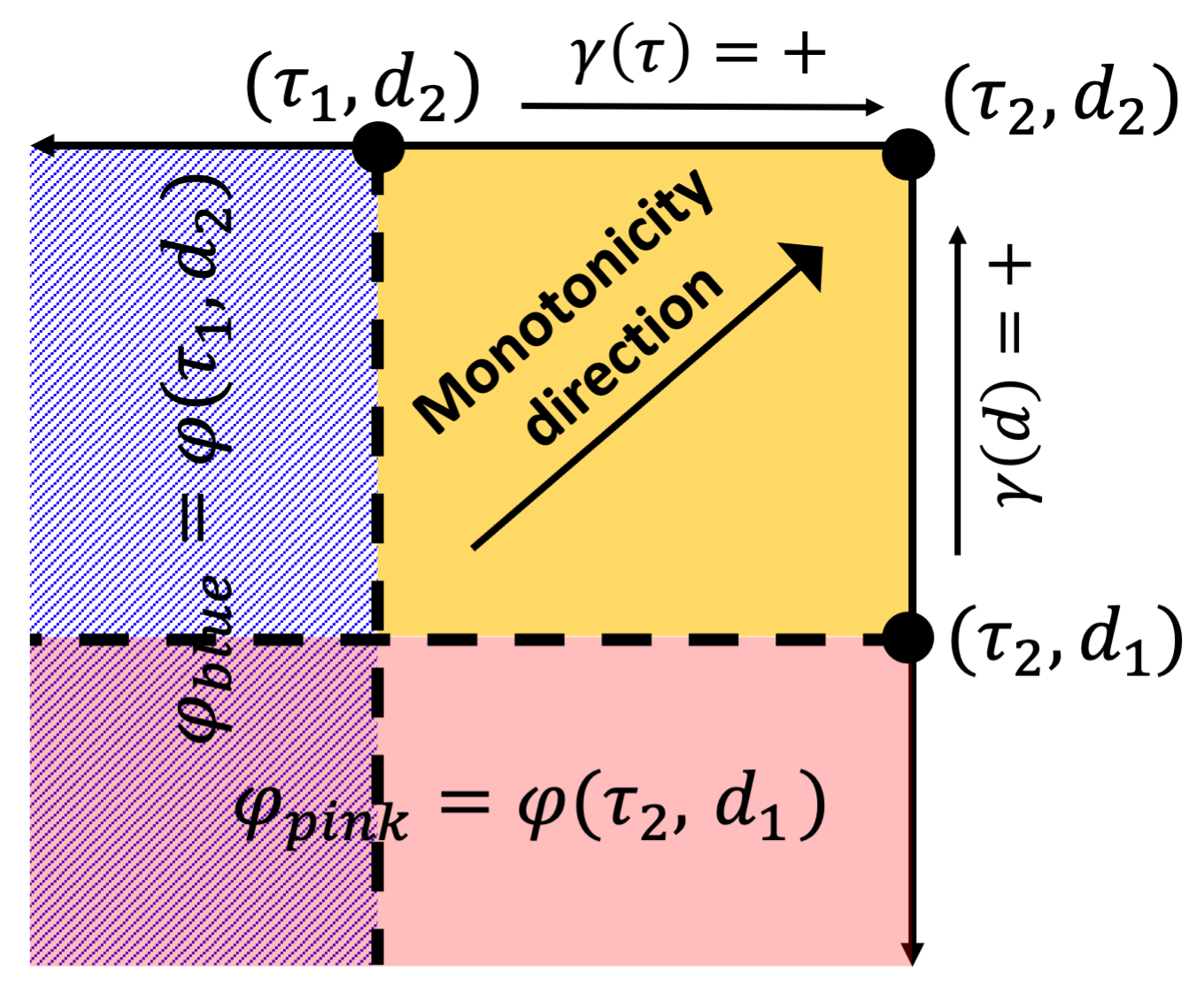}}
\ 
\subfloat[The result of applying decision tree algorithm on labeled parameter valuations shown in Fig.~\ref{fig:bss_clustering_runex}\label{fig:dt_branches}]{\includegraphics[width = .4\textwidth]{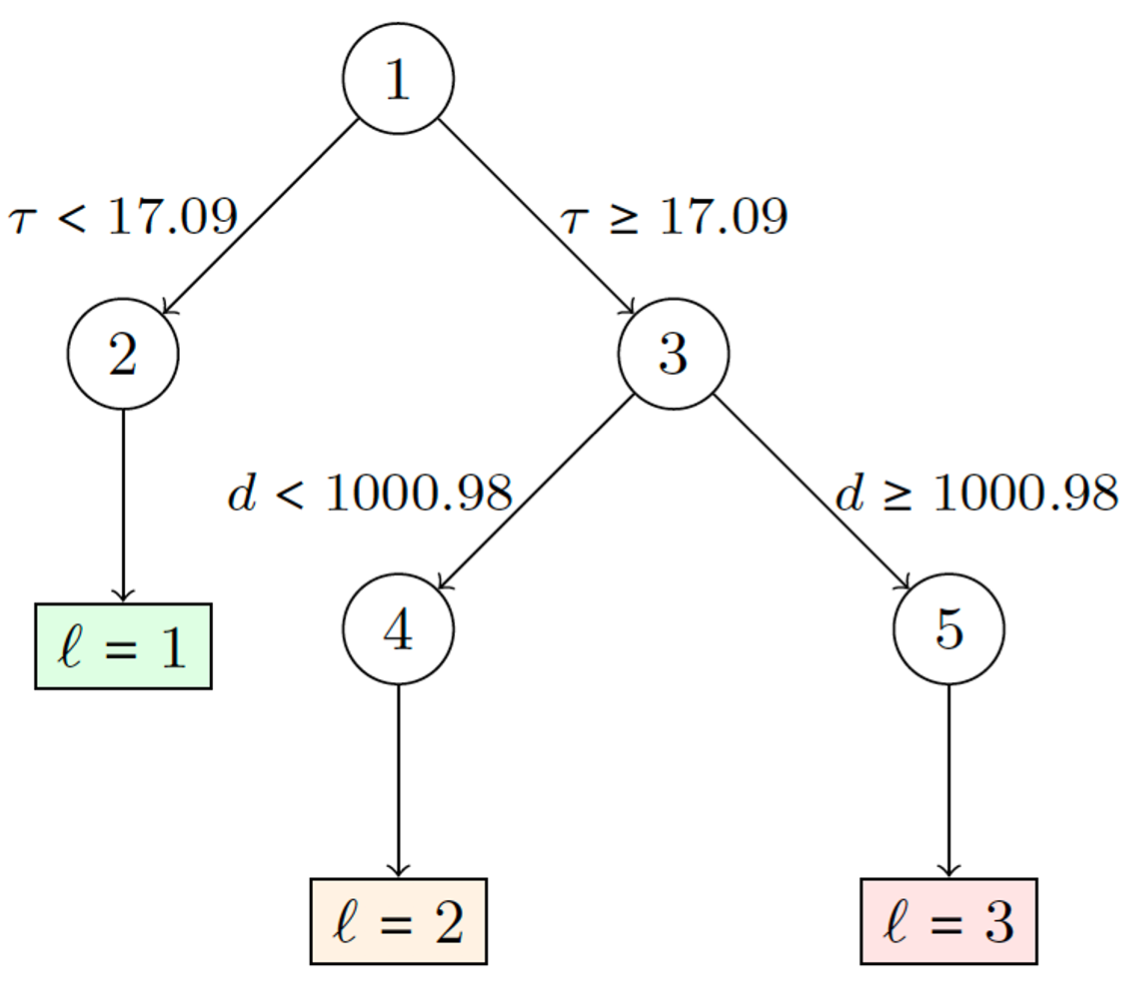}}
\caption{Illustration of clustering on the BSS locations}
\end{figure}

\begin{algorithm}
\KwIn{The learned projections $\pi$, labeling function $\clusterlabel$, threshold on the maximum depth of the Decision Tree $N$, parameter $K$ for k-fold cross validation method, threshold on the accuracy $ACC_{th}$}
\KwOut{The learned decision tree $DT$} 
\SetKwProg{Fn}{Function}{}{}
    {\color{teal}\tcp{Loop over the maximum depth of Decision Tree}}
    \For{$d \leftarrow 1$ to $N$}{
        {\color{teal}\tcp{Compute the cross validation accuracy for $maxDepth = d$}}
        $ACC_d= kfoldCrossValidation(\pi,\clusterlabel(\pi), maxDepth=d, K)$\;
        {\color{teal}\tcp{Choose the max depth that gives the best cross validation accuracy}}
        \If{$ACC_d > ACC_{th}$}{
            {\color{teal}\tcp{Train a Decision Tree with the chosen max depth}}
            $DT = fitDecisionTree(\pi, \clusterlabel(\pi), d)$\;
            \Return $DT$\;
        }

    }
    \Return $\emptyset$;
    
    \Fn{$kfoldCrossValidation(X, Y, maxDepth, K)$}{
        {\color{teal}\tcp{Shuffle the data}}
        $X, Y = Shuffle(X, Y)$\;
        {\color{teal}\tcp{Devide the data into K subsets}}
        $X(1:K), Y(1:K) = DivideToKSubsets (X, Y)$\;
        $sumACC = 0$\;
        {\color{teal}\tcp{Train on K-1 subsets and test on 1 subset}}
        \For{$i \leftarrow 1$ to $K$}{
            $X_{train},Y_{train} = [X(1:i-1), X(i+1:K)], [Y(1:i-1), Y(i+1:K)]$\;
            $X_{test}, Y_{test} = X(i), Y(i)$\;
            $DT = fitDecisionTree(X_{train}, Y_{train}, maxDepth)$\;
            $ACC = predictDecisionTree(DT, X(i), Y(i))$\;
            $sumACC = sumACC + ACC$\;
        }
        {\color{teal}\tcp{Return the average accuracy}}
        \Return $sumACC/K$ \;
    }
\caption{K-fold cross validation approach to determine the best maximum depth of the Decision Tree\label{alg:dt_prune}}
\end{algorithm}

\section{Case Studies}
\label{sec:experiments}
We now present the results of applying the clustering techniques
developed on three benchmarks: (1) COVID-19 data from Los Angeles
County, USA, (2) Outdoor Air Quality data from California, and (3) BSS
data from the city of Edinburgh (running example)\footnote{We provide
results on a fourth benchmark consisting of a synthetic dataset for
tracking movements of people in a food court building and detailed
descriptions for each benchmark in the appenidx.  All experiments were
performed on an Intel Core-i7 Macbook Pro with 2.7 GHz processor and
16 GB RAM. We use an existing monitoring tool MoonLight
\cite{bartocci2020moonlight} in Matlab for computing the robustness of
STREL formulas. For Agglomerative Hierarchical Clustering and Decision
Tree techniques we use scikit-learn library in Python and the
Statistics and Machine Learning Toolbox in Matlab.}. A summary of the
computational aspects of the results is provided in
Table.~\ref{tab:results}. The numbers indicate that our methods 
scale to spatial models containing hundreds of locations, and still
learn interpretable STREL formulas for clusters.
\begin{table}[ht]
\centering
\begin{tabular*}{.95\textwidth}{@{\extracolsep{\fill}}llllll}
\toprule
% \hline
Case & $|\locations|$ & $|\wfun|$ & run-time (secs) & $\numclusters$ & $|\varphi_{\mathit{cluster}}|$  \\
% \hline
\midrule
COVID-19&235&427&813.65&3& $3 \cdot |\varphi| + 4$\\
BSS&61&91&681.78&3& 2 $\cdot |\varphi| + 4$\\
Air Quality&107&60&136.02&8& $5 \cdot |\varphi|+7$\\
Food Court* &20&35&78.24&8& $3 \cdot |\varphi| + 4$ \\ [1ex]
\bottomrule
\end{tabular*}
\caption{Summary of results.\label{tab:results}}
\end{table}

\mypara{COVID-19 data from LA County} Understanding the spread pattern
of COVID-19 in different areas is vital to stop the spread of the
disease. While this example is not related to a software system, it is
nevertheless a useful example to show the versatility of our approach
to spatio-temporal data.  The PSTREL formula $\varphi(c,d) =
\somewhere{[0,d]}\{\ev{[0,\tau]}(x > c)$ allows us to number of cases
exceeding a threshold $c$ within $\tau=10$ days in a neighborhood of
size $d$ for a given location\footnote{ We fix $\tau$ to $10$ days and
focus on learning the values of $c$ and $d$ for each location.}.
Locations with small value of $d$ and large value of $c$ are unsafe as
there is a large number of new positive cases within a small radius
around them. 

% However, regions with large values of $d$ and small values of $c$ are potentially safe regions because within a large radius of such regions there is only a small number of new positive cases. 
% To learn the tight values of $c$ and $d$ for each region in LA County we use the Bisection Search approach formalized in Algo.~\ref{alg:bisection_search} (assuming that $\tau$ is fixed to 10 days). This technique projects each region in LA county to a representative point in the parameter space of $\varphi$. Next, we use Agglomerative Hierarchical Clustering to cluster the regions into multiple groups with respect to $\varphi$. We use the learned parameter valuations $(c, d)$ as features for clustering algorithm with the number of clusters as $k=3$. 

We illustrate the clustering results in Fig.~\ref{fig:covid_example1}.
Each location in Fig.~\ref{fig:covid_exampl1_clustering} is associated
with a geographic region in LA county (shown in
Fig.~\ref{fig:covid_exampl1_map}), and the \textit{red} cluster
corresponds to hot spots (small $d$ and large $c$).  Applying the DT
classifier on the learned clusters (shown in
Fig.~\ref{fig:covid_exampl1_clustering}) produces 11 hyperboxes, some
of which contain only a few points. Hence we apply our DT pruning
procedure to obtain the largest cluster that gives us at least $90\%$
accuracy. Fig.~\ref{fig:covid_exampl1_pruned} shows the results after
pruning the Decision Tree. We learn the following formula:
\[
\varphi_{red} = \somewhere{[0,4691.29]}(\ev{[0,10]}(x > 3180)) \vee
\somewhere{[0,15000]}(\F_{[0,10]}(x > 5611.5)), \]
This formula means that within 4691.29 meters from any \textit{red}
location, within 10 days, the number of new positive cases exceeds
3180. The COVID-19 data that we used is for September 2020\footnote{In
Fig.~\ref{fig:cluster_changing_time} in the appendix, we show the
results of STREL clustering for 3 different months in 2020, which
confirms the rapid spread of the COVID-19 virus in LA county from
April 2020 to September 2020. Furthermore, we can clearly see spread
of the virus around the hot spots during the time, a further
validation of our approach.}.

\begin{figure}[t]
\centering
%  \setkeys{Gin}{height=3cm} % just for the example
%\makebox[\textwidth]{%
%   \setlength{\tabcolsep}{3pt}%
%   \begin{tabular}{@{}cc@{}}
\subfloat[The learned hyper-boxes before pruning the DT.
\label{fig:covid_exampl1_clustering}]{\includegraphics[width = .45\textwidth]{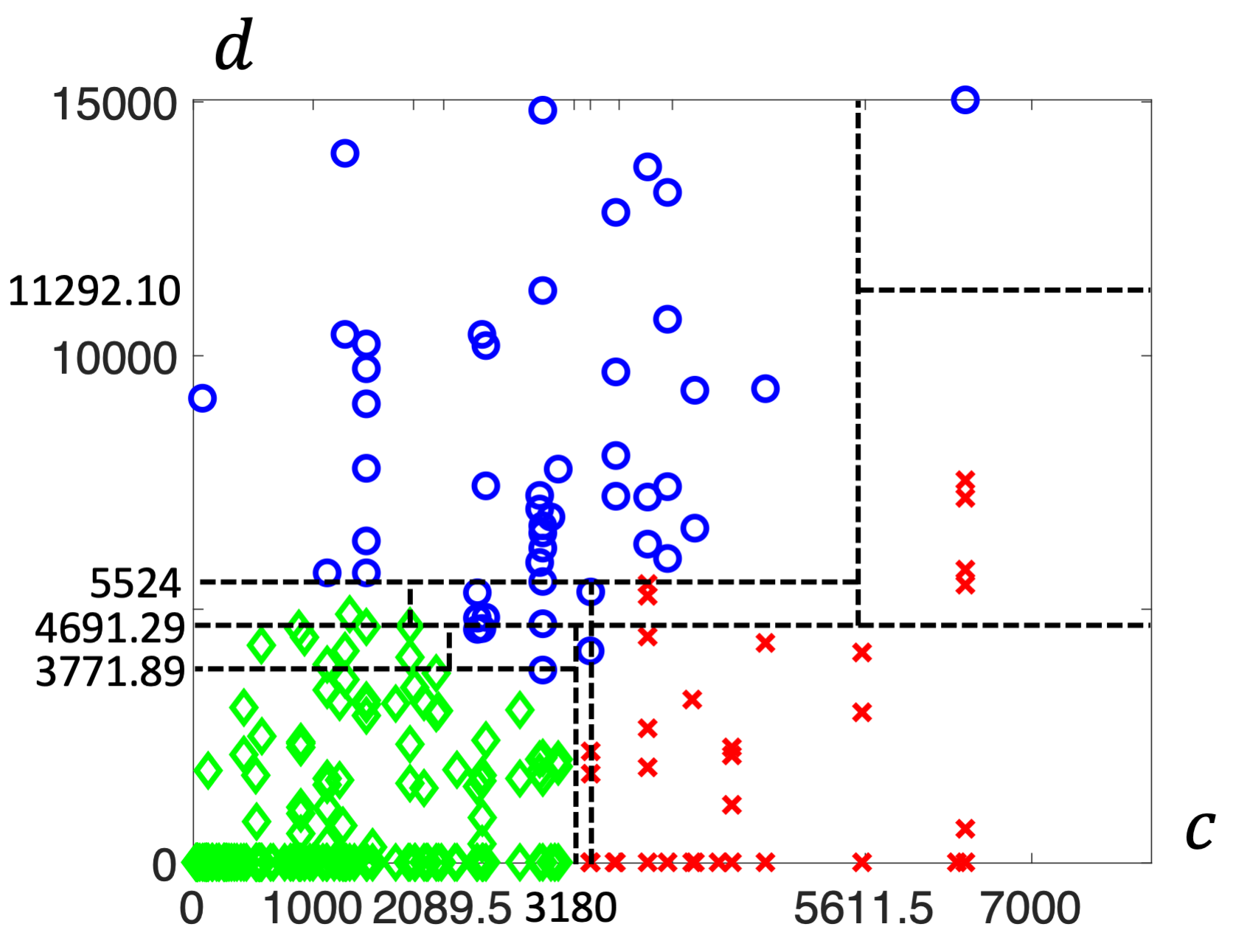}} \ 
\subfloat[The learned hyper-boxes after pruning the DT.\label{fig:covid_exampl1_pruned}]{\includegraphics[width =
.45\textwidth]{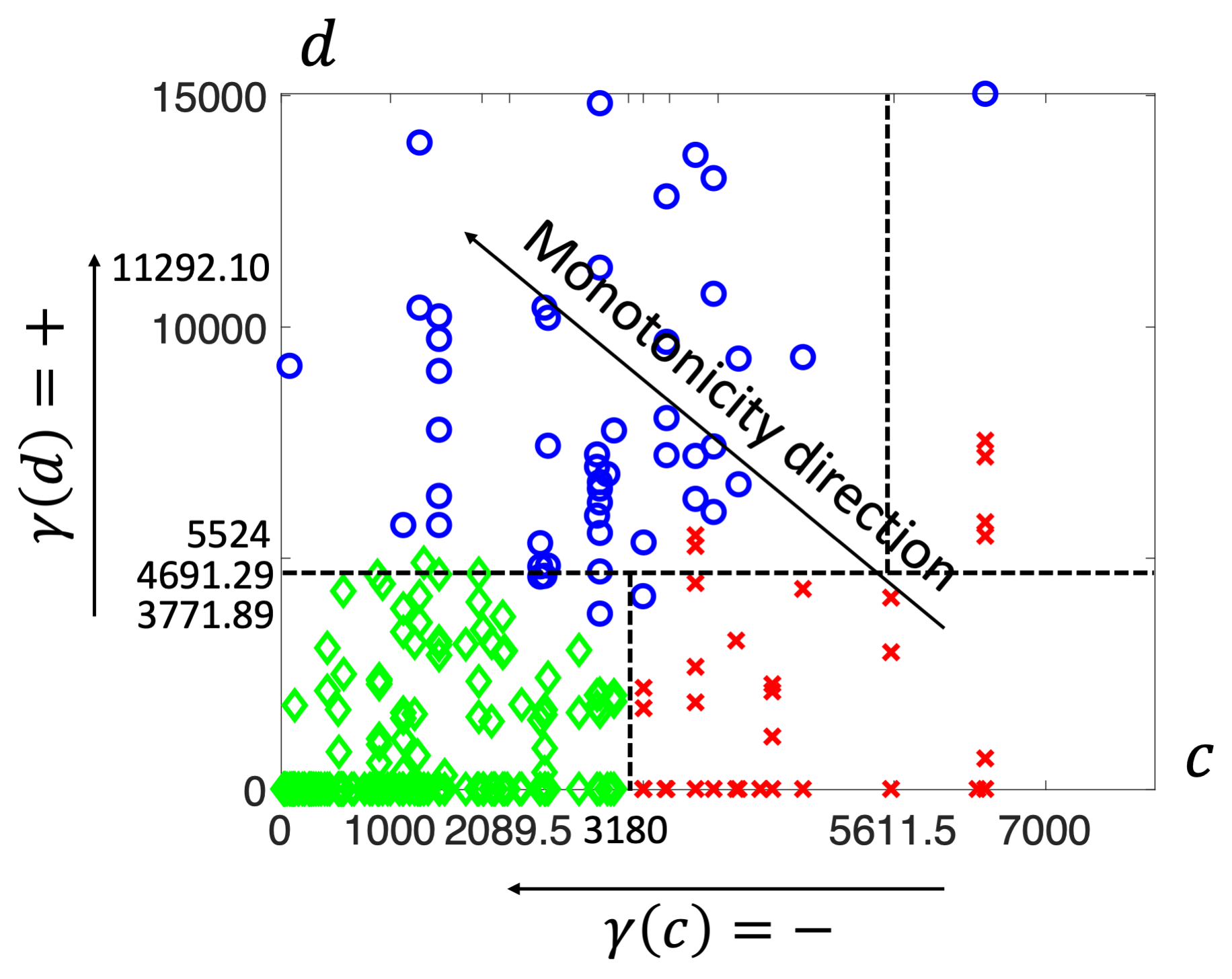}} \ 
\subfloat[Red-color points: hot spots.
\label{fig:covid_exampl1_map}]{\includegraphics[width = .45\textwidth]{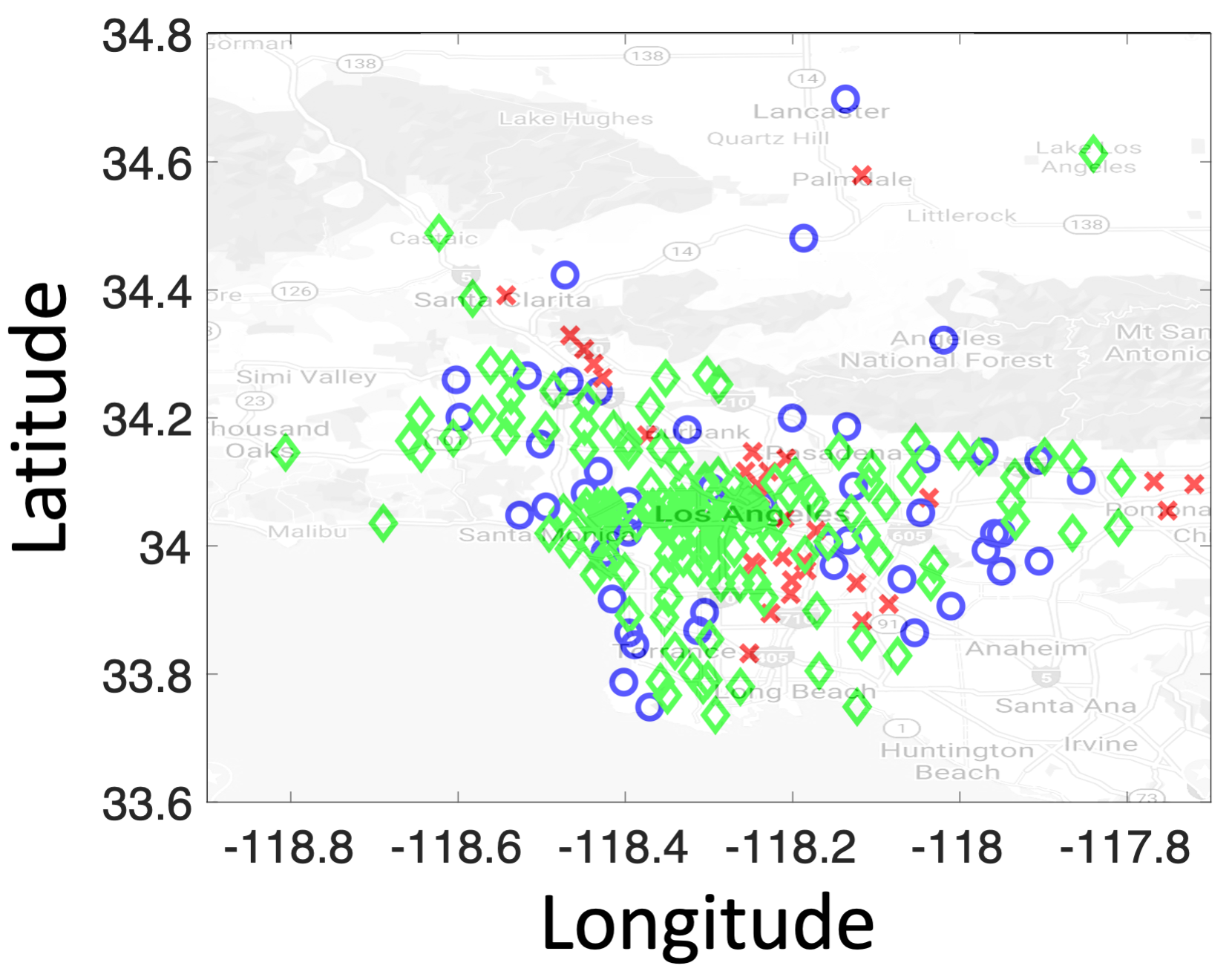}}
% \end{tabular}%
% }
\caption{Procedure to learn STREL formulas from COVID-19 data with
template PSTREL formula $\varphi(c,d) = \somewhere{[0,d]}{}(\ev{[0,10]}(x > c))$
\label{fig:covid_example1}}
\end{figure}

\mypara{Outdoor Air Quality data from California} We next consider Air
Quality data from California gathered by the US Environmental
Protection Agency (EPA).  Among reported pollutants we focus on $PM2.5$
contaminant, and try to learn the patterns in the amount of $PM2.5$ in
the air using STREL formulas. Consider a mobile sensing network
consisting of UAVs to monitor pollution, such a STREL formula could be
used to characterize locations that need increased monitoring.

We use the PSTREL formula 
$\varphi(c,d)$ = $\glob{[0,10]}(\escape{[d,16000]}{}(PM2.5 < c))$
and project each
location in California to the parameter space of $c,d$. 
A location
$\ell$ satisfies this property if it is always true within the next 10
days, that there exists a location $\ell'$ at a distance more than
$d$, and a route $\tau$ starting from $\ell$ and reaching $\ell'$ such
that all the locations in the route satisfy the property $PM2.5 < c$.
Hence, it might be possible to escape to a location at a distance
greater than $d$ always satisfying property $PM2.5 < c$. 
The results are shown in Fig.~\ref{fig:pollution_exampl2_clustering}. 
Cluster 8 is the best cluster as it has a
small value of $c$ and large value of $d$ which means that there
exists a long route from the locations in cluster 8 with low density of
$PM2.5$. Cluster 3 is the worst as it has a large value of $c$
and a small value of $d$. The formula for cluster $3$ is 
$\varphi_3$ = $\varphi(500,0) \wedge \neg \varphi(500,2500) \wedge
\neg \varphi(216,0)$.
% = G_{[0,10]}\{\escape{[0,16000]}{}(PM2.5 < 500)\} \wedge \\& \neg G_{[0,10]}\{\escape{[2500,16000]}{}(PM2.5 < 500)\} \wedge \neg G_{[0,10]}\{\escape{[0,16000]}{}(PM2.5 < 216)\} 
$\varphi_3$ holds in locations where, in the next 10 days, $PM2.5$ is
always less than $500$, but at least in 1 day $PM2.5$ reaches $216$ and
there is no safe route (i.e. locations along the route have $PM2.5$ <
500) of length at least 2500.

% {\centering
%   $ \displaystyle
%     \begin{aligned} 
% % = 
% % \\& G_{[0,10]}\{\escape{[0,16000]}{f}(PM2.5 < 500)\} \wedge  F_{[0,10]}\{\neg \escape{[2500,16000]}{f}(PM2.5 < 500)\} \wedge F_{[0,10]}(PM2.5 \geq 216)
% \\ 
% \vspace{1mm}
% &
% \varphi_8 = \varphi(83.5,12687.5) \wedge \neg \varphi(0,12687.5) \wedge \neg \varphi(83.5,15000) = \varphi(83.5,12687.5) = \\& G_{[0,10]}\{\escape{[12687.5,16000]}{}(PM2.5 < 83.5)\}
%     \end{aligned}
%   $ 
% \par}
% \vspace{3mm}
% 
% \begin{align*}
% &\varphi_3 = \varphi(500,0) \wedge \neg \varphi(500,2500) \wedge \neg \varphi(216,0) = G_{[0,10]}\{\escape{[0,16000]}{f}(PM2.5 < 500)\} \wedge \\& \neg G_{[0,10]}\{\escape{[2500,16000]}{f}(PM2.5 < 500)\} \wedge \neg G_{[0,10]}\{\escape{[0,16000]}{f}(PM2.5 < 216)\} 
% % = 
% % \\& G_{[0,10]}\{\escape{[0,16000]}{f}(PM2.5 < 500)\} \wedge  F_{[0,10]}\{\neg \escape{[2500,16000]}{f}(PM2.5 < 500)\} \wedge F_{[0,10]}(PM2.5 \geq 216)
% \\ &
% \varphi_8 = \varphi(83.5,12687.5) \wedge \neg \varphi(0,12687.5) \wedge \neg \varphi(83.5,15000) = \varphi(83.5,12687.5) = \\& G_{[0,10]}\{\escape{[12687.5,16000]}{f}(PM2.5 < 83.5)\}
% \end{align*}

% The formula $\varphi_8$ holds in all the locations where it is  always true, within the next 10 days, that there exists a route reaching a location at a distance greater than 12687.5 m from the the current one, such that all the locations in the route have $PM2.5$ less than $83.5$. 

\begin{figure}[!tp]
\centering
% \setkeys{Gin}{height=3cm} 
% \makebox[\textwidth]{%
%   \setlength{\tabcolsep}{4pt}%
%   \begin{tabular}{@{}cc@{}}
      \subfloat[The learned Hyper-boxes from Air Quality data. \label{fig:pollution_exampl2_clustering}]{\includegraphics[width = .45\textwidth]{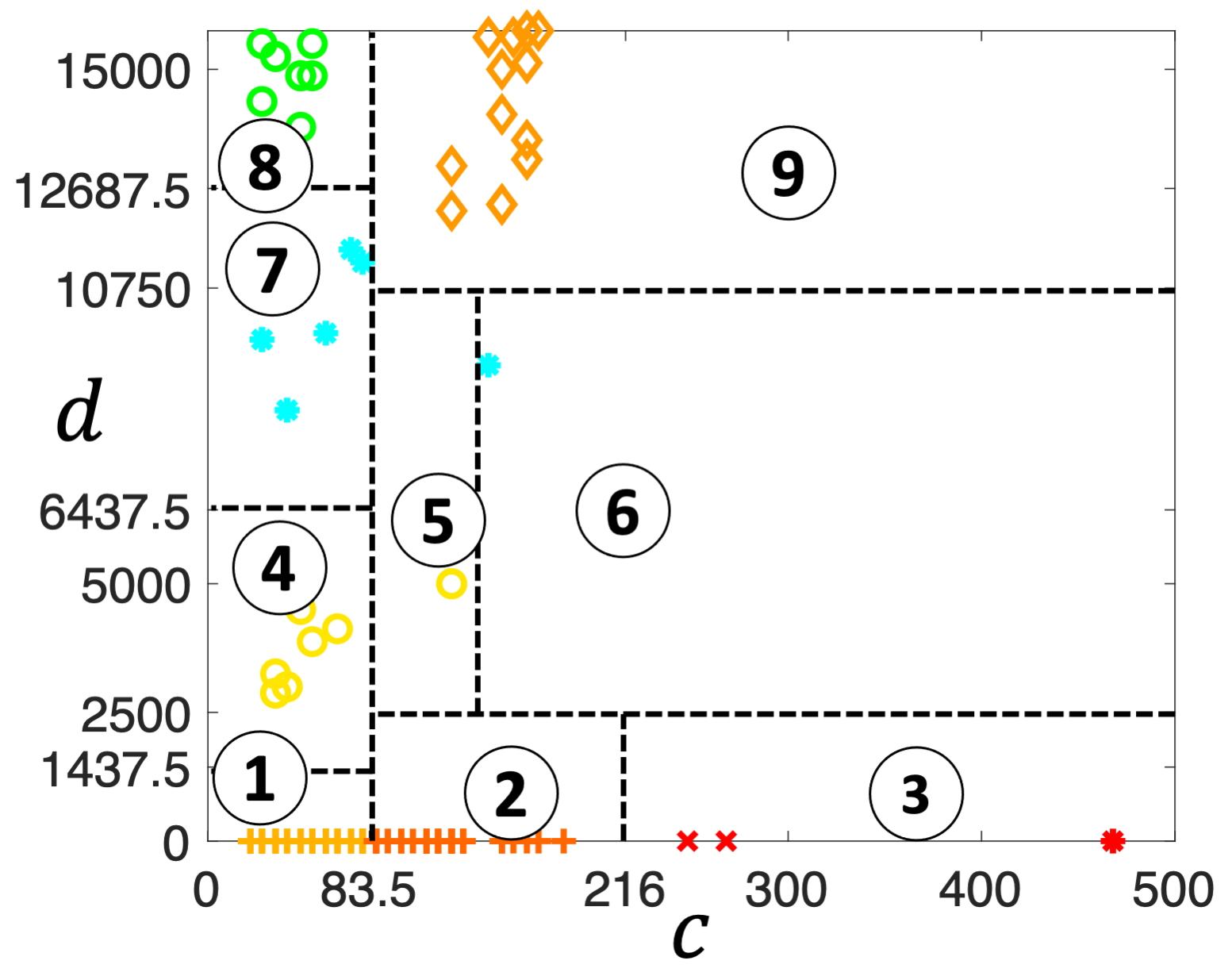}} \quad
  \subfloat[Red and orange points: high density of $PM2.5$. \label{fig:pollution_exampl2_map}]{\includegraphics[width = .45\textwidth]{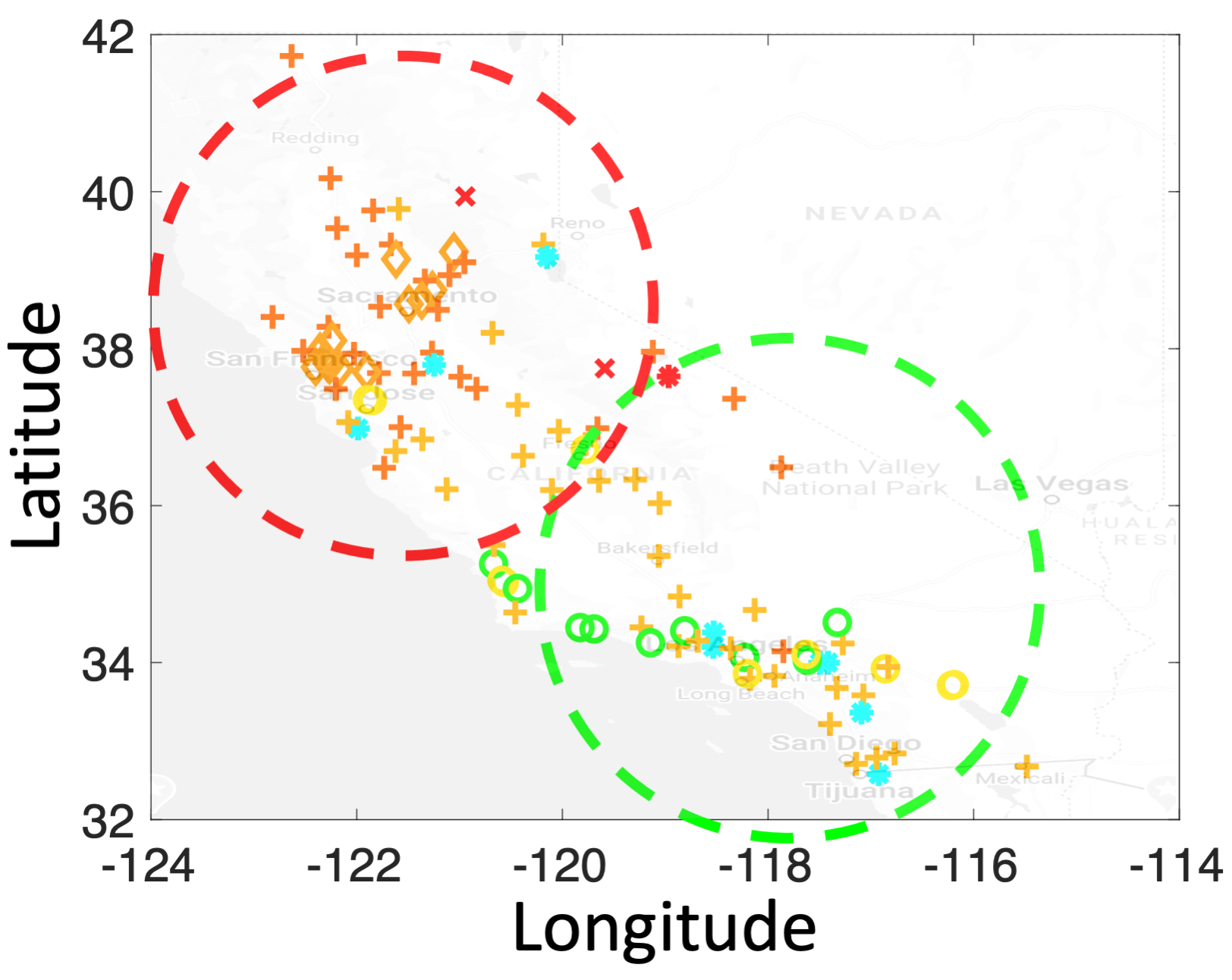}}
%   \end{tabular}%
% }
\caption{Clustering experiments on the California Air Quality Data}
\label{fig:pollution}
\end{figure}

\section{Related work and Conclusion}
\label{sec:related_work}
\mypara{Traditional ML approaches for time-series clustering}
Time-series clustering is a popular area in the domain of machine learning and data mining. Some techniques for time-series clustering combine clustering methods such as KMeans, Hierarchical Clustering, agglomerative clustering and etc., with similarity metrics
between time-series data such as the Euclidean distance, dynamic
time-warping (DTW) distance, and statistical measures (such as mean,
median, correlation, etc). Some recent works such as the works on shapelets automatically identify distinguishing shapes in the time-series data \cite{zakaria2012clustering}. Such shapelets serve as features for ML tasks. All these approaches are based on shape-similarity which might be useful in some applications; however, for applications that the user is interested in mining temporal information from data, dissimilar traces might be clustered in the same group \cite{vazquez2017logical}. Furthermore, such approaches may lack interpretablility as we showed in BSS case study.

\mypara{STL-based clustering of time-series data} 
There is considerable amount of recent work on learning temporal logic formulas from time-series data using logics such as Signal Temporal Logic (STL) \cite{vazquez2017logical,jin2015mining,mohammadinejad2020interpretable,mohammadinejad2020mining}; however, there is no work on discovering such relations on spatio-temporal data.
In particular, the work in \cite{vazquez2017logical} which addresses unsupervised clustering of time-series data using Signal Temporal Logic is closest to our work.
% In \cite{vazquez2017logical}, the authors assume that the user provides a Parametric Signal Temporal Logic (PSTL) formula, and the procedure projects given temporal data onto the parameter domain of the PSTL formula. 
% The authors use off-the-shelf clustering techniques to group parameter values and identify STL formulas corresponding to each cluster. 
There are a few hurdles in applying such an approach to spatio-temporal data as explained in Section.~\ref{sec:intro}. 
% First, in \cite{vazquez2017logical}, the authors assume a monotonic fragment of PSTL: there is no such fragment identified in the literature for STREL. Second, in \cite{vazquez2017logical}, the authors assume that clusters in the parameter space can be separated by using a tight axis-aligned hyper-box for each cluster. Third, given spatio-temporal data, we can have different choices to impose the edge relation on nodes, which can affect the formula we learn. 
We address all the hurdles in the current work.

\mypara{Monitoring spatio-temporal properties}
There is considerable amount of recent work  such as \cite{bartocci2017monitoring,bartocci2020moonlight} on monitoring spatio-temporal properties. Particularly, MoonLight \cite{bartocci2020moonlight} is a recent tool for monitoring of STREL properties, and in our current work, we use MoonLight for computing the robustness of spatio-temporal data with respect to STREL formulas. MoonLight uses $(\delta,d)$-connectivity approach for creating a spatial model, which has several issues, including dis-connectivity and distance overestimation. We resolve these issues by proposing our new method for creating the spatial graph, which we call Enhanced MSG.
While there are many works on monitoring of spatio-temporal logic, to the best of our knowledge, there is no work on automatically inferring spatio-temporal logic formulas from data that we address in this work.

\label{sec:conc}
\mypara{Conclusion} In this work, we proposed a technique to learn interpretable STREL formulas from spatio-temporal time-series data for Spatially Distributed Systems. First, we introduced the notion of monotonicity for a PSTREL formula, proving the monotonicity of each spatial operator. We proposed a new method for creating a spatial model with a restrict number of edges that preserves connectivity of the spatial model. We leveraged quantitative semantics of STREL combined with multi-dimensional bisection search to extract features for spatio-temporal time-series clustering. We applied Agglomerative Hierarchical clustering on the extracted features followed by a Decision Tree based approach to learn an interpretable STREL formula for each cluster. We then illustrated with a number of
benchmarks how this technique could be used and the kinds of insights it can
develop. The results show that while our method performs slower than traditional ML approaches, 
it is more interpretable and provides a better insight into the data. For future work, we will study extensions of this approach to supervised and active learning.

\mypara{Acknowledgments}
We thank the anonymous reviewers for their comments. The authors also gratefully acknowledge the support by the National Science Foundation under the Career Award SHF-2048094 and the NSF FMitF award  CCF-1837131, and a grant from Toyota R\&D North America.

\bibliographystyle{unsrtnat}
\bibliography{references}  

\newpage
\appendix
\section*{Appendix}
\section{Boolean and Quantitative Semantics of STREL}
STREL is equipped with both Boolean and quantitative semantics; a Boolean semantics, $(\spatialmodel, \sts, \ell, t) \models \varphi$, with the
meaning that the spatio-temporal trace $\sigma$ in location $\ell$ at time $t$ with spatial
model $\mathcal{S}$, satisfies the formula $\varphi$ and a quantitative semantics, $\rho(\varphi, \mathcal{S}, \sigma, t)$, that can be used to measure the quantitative level of satisfaction of a formula for a given trajectory and space model. The function $\rho$ is also called the robustness function. The satisfaction of the whole trajectory correspond to the satisfaction at time $0$, i.e. $\rho(\varphi, \mathcal{S}, \sigma) = \rho(\varphi, \mathcal{S}, \sigma, 0)$. Semantics for Boolean and temporal operator remain the same as STL \cite{maler2004monitoring}. We describe below the 
quantitative semantics of the spatial operators. Boolean semantics can be derived substituting $\min, \max$ with $\vee, \wedge$ and considering the Boolean satisfaction instead of $\rho$.
\paragraph{\bf Reach} 
The quantitative semantics of the reach operator is: 

\begin{align*}
&\rho(\varphi_{1} \: \reach{[d_1,d_2]}{} \: \varphi_{2}, \mathcal{S}, \sigma,\ell, t)= \max_{\tau\in \routes(\spatialmodel,\ell)}\max_{\ell'\in \tau:\left(\routedistance^\tau(\ell') \in [d_1,d_2] \right)} 
     \linebreak   ( \min( \rho(\varphi_2, \mathcal{S}, \sigma, \ell', t), \\&\min_{j < \tau(\ell')}  \rho(  \varphi_1, \mathcal{S}, \sigma,  \tau[j], t )))    
\end{align*}

$(\mathcal{S}, \sigma, \ell,t)$, a spatio-temporal trace $\sigma$, in location $\ell$, at time $t$, with a spatial model $\mathcal{S}$, satisfies
$\varphi_{1} \: \reach{[d_1, d_2]}{} \: \varphi_{2}$ iff it satisfies $\varphi_2$ in a location $\ell'$ reachable from $\ell$ through a route $\tau$, with a length $\routedistance^\tau(\ell') \in [d_1, d_2]$, and such that $\tau[0]=\ell$ and all its elements with index less than $\tau(\ell')$ satisfy $\varphi_1$.  
Intuitively, the reachability operator $\varphi_1 \reach{[d_1, d_2]}{} \varphi_2$ describes the behavior of reaching a location satisfying property $\varphi_2$ passing only through locations that satisfy $\varphi_1$, and such that the distance from the initial location and the final one belongs to the interval $[d_1, d_2]$.

\paragraph{\bf Escape} 
The quantitative semantics of the escape operator is:$$\rho( \escape{[d_1,d_2]}{} \: \varphi, \mathcal{S}, \sigma,\ell, t)= \max_{\tau\in \routes(\mathcal{S},\ell)} 
        \max_{\ell'\in \tau:\left(\routedistance^\tau[\ell,\ell'] \in [d_1, d_2]\right) }\linebreak
            ~\min_{i \leq \tau(\ell')} 
                \rho(\varphi, \mathcal{S}, \sigma, \tau[i], t).
$$

\noindent
$(\mathcal{S}, \sigma, \ell,t)$, a spatio-temporal trace $\sigma$, in location $\ell$, at time $t$, with a spatial model $\mathcal{S}$, satisfies $\escape{[d_1,d_2]}{} \: \varphi$ if and only if there exists a route $\route$ and a location $\ell'\in \route$ such that $\route[0]=\ell$,    $\routedistance^\tau[\tau[0],\ell'] \in [d_1, d_2]$ and all elements $\tau[0],...\tau[k]$ (with $\route(\ell')=k$) satisfy $\varphi$.  
Practically, the escape operator $\escape{[d_1, d_2]}{} \varphi$,
instead, describes the possibility of escaping from a certain region
passing only through locations that satisfy $\varphi$, via a route with a distance that belongs to the interval $[d_1, d_2]$.
\paragraph{\bf Somewhere}   $\somewhere{[d_1,d_2]}{} \varphi := true  \reach{[d_1,d_2]}{} \varphi$ holds for $(\mathcal{S}, \sigma, \ell,t)$ iff there exists a location $\ell'$ in $\mathcal{S}$ such that $(\mathcal{S}, \sigma, \ell',t)$ satisfies $\varphi$ and $\ell'$ is reachable from $\ell$ via a route $\tau$ with length $\routedistance^\tau[\ell'] \in [d_1, d_2]$.

\paragraph{\bf Everywhere.}  $ \everywhere{[d_1, d_2]}{} \varphi := \neg \somewhere{[d_1,d_2]}{} \neg \varphi $ 
holds for $(\mathcal{S}, \sigma, \ell,t)$ iff  all the locations $\ell'$ reachable from $\ell$ via a path,with length $\routedistance^\tau[\ell'] \in [d_1, d_2]$, satisfy $\varphi$.

\paragraph{\bf Surround} $\varphi_{1} \surround{ [d_1,d_2] }{} \varphi_{2} := \varphi_{1} \wedge \neg (\varphi_{1}\reach{[d_1,d_2]}{} \neg (\varphi_1 \vee \varphi_{2})) \wedge \neg (\escape{[d_2,\infty]}{}  (\varphi_{1})) $ holds for $(\mathcal{S}, \sigma, \ell,t)$ iff there exists a $\varphi_{1}$-region that contains $\ell$, all locations in that region satisfies $\varphi_{1}$ and  are reachable from $\ell$ through a path with length less than $d_2$. Furthermore, all the locations that do not belong to the $\varphi_{1}$-region but are directly connected to a location in $\varphi_{1}$-region must satisfy $\varphi_2$ and be reached from $\ell$ via a path with length in the interval $[d_1, d_2]$. 
Intuitively,  the surround operator indicates the notion of being surrounded by a $\varphi_2$-region, while being in a $\varphi_{1}$-region, with some added constraints.
The idea is that one cannot escape from a $\varphi_{1}$-region without passing from a node that satisfies $\varphi_2$ and, in any case, one has to reach a $\varphi_2$-node at a distance between $d_{1}$ and $d_{2}$.

\begin{lemma}
\label{lem:monitoring_complexity}
Let $\langle \locations, \wfun \rangle$ be a spatial model where
${\routedistance}_{\mathrm{min}}$ is the minimum distance between two
locations and let $H_{[d_1,d_2]}$ be a STREL formula where $H$ is an
arbitrary spatial operator.  Then, the complexity of monitoring
formula  is  $O(k^2 \cdot |\locations| \cdot |\wfun|)$, where
$k=\min\{ i | i\cdot {\routedistance}_{\mathrm{min}} > d_{2} \}$.
\end{lemma}

\section{Monotonicity Proofs for spatial operators}

\begin{lemma}
The polarity for PSTREL formulas
$\varphi(d_1,d_2)$ of the form $\psi_1\reach{[d_1,d_2]}\psi_2$, $\escape{[d_1,d_2]}\psi$, $\somewhere{[d_1,d_2]}\psi$ and $\psi_1\surround{[d_1,d_2]}\psi_2$ are
$\polarity(d_1) = -$ and $\polarity(d_2) = +$, i.e. if a
spatio-temporal trace satisfies $\varphi(\val(d_1),\val(d_2))$, then
it also satisfies any STREL formula over a strictly larger spatial
model induced distance interval, i.e. by decreasing $\val(d_1)$ and
increasing $\val(d_2)$.  For a formula $\everywhere{[d_1,d_2]}\psi$,
$\polarity(d_1) = +$ and $\polarity(d_2) = -$, i.e. the formula
obtained by strictly shrinking the distance interval.
\end{lemma}

\begin{proof}
To prove the above lemma, we first define some ordering on intervals. For intervals $I=[a,b]$ and $I^\prime=[a^\prime,b^\prime]$,  $$I^\prime \geq I \iff a^\prime \leq a \;and\; b^\prime \geq b.$$

Followed by the defined ordering on intervals, 
\begin{equation}
\label{eqn:max_i}
\max_{I^\prime} f(x) \geq \max_I f(x) \iff I^\prime \geq I
\end{equation}
\begin{equation}
\label{eqn:min_i}
\min_{I^\prime} f(x) \leq \min_I f(x) \iff I^\prime \geq I
\end{equation}

Assuming $d_1^\prime \leq d_1$ and $d_2^\prime \geq d_2$, from quantitative semantics of the Reach operator and equation.~\ref{eqn:max_i} we get:

$$\rho(\varphi_{1} \: \reach{[d_1^\prime,d_2]}{} \: \varphi_{2}, \mathcal{S}, \sigma,\ell, t) \geq \rho(\varphi_{1} \: \reach{[d_1,d_2]}{} \: \varphi_{2}, \mathcal{S}, \sigma,\ell, t),$$

$$\rho(\varphi_{1} \: \reach{[d_1,d_2^\prime]}{} \: \varphi_{2}, \mathcal{S}, \sigma,\ell, t) \geq \rho(\varphi_{1} \: \reach{[d_1,d_2]}{} \: \varphi_{2}, \mathcal{S}, \sigma,\ell, t),$$

which proves that the Reach operator is monotonically decreasing with respect to $d_1$ and monotonically increasing with respect to $d_2$. The proofs for other spatial operators are similar, and we skip for brevity.

% $$\max_{\tau\in Routes(\mathcal{S}(t),\ell)}\max_{\ell'\in \tau:\left(d_{\tau}^{f}[\ell'] \in [d_1,d_2^\prime] \right)} 
%      \linebreak f(x) \geq \max_{\tau\in Routes(\mathcal{S}(t),\ell)}\max_{\ell'\in \tau:\left(d_{\tau}^{f}[\ell'] \in [d_1,d_2] \right)} 
%      \linebreak f(x).$$
\end{proof}

\section{Case studies}
\subsection{COVID-19 data from LA County}
COVID-19 pandemic has affected everyone's life tremendously. Understanding the spread pattern of COVID-19 in different areas can help people to choose their daily communications appropriately which helps in decreasing the spread of the virus. Particularly, we consider COVID-19 data from LA County \cite{kiamari2020covid} and use STREL formulas to understand the underlying patterns in the data. The data consists of the number of daily new cases of COVID-19 in different regions of LA County. In our analysis, we exclude the regions with more than $15\%$ missing values resulting in a total of 235 regions. For the remaining 235 regions, we fill the missing values with the nearest non-missing value. We construct the spatial model using our $(\alpha,\dhaver)$-Enhanced MSG approach with $\alpha=2$ which results in a total of 427 edges in the spatio-temporal graph. 
One of the important properties for a region is: ``at some place within the radius of $d$ from the region, eventually within $\tau$ days, the number of new COVID-19 cases goes beyond a certain threshold c''. For simplicity, we fix $\tau$ to $10$ days, and we are interested in learning the tight values of $c$ and $d$ for each region in LA County. Regions with small value of $d$ and large value of $c$ are unsafe regions because there is a large number of new positive cases within a small radius around such regions. However, regions with large values of $d$ and small values of $c$ are potentially safe regions because within a large radius around such regions there is only a small number of new positive cases. This property can be specified using PSTREL formula $\varphi(c,d) = \somewhere{[0,d]}{}\{\F_{[0,\tau]}(x > c)$. To learn the tight values of $c$ and $d$ for each region in LA County we use the lexicographic projection approach formalized in Algo.~\ref{alg:bisection_search} (assuming that $\tau$ is fixed to 10 days and ordering of $c >_\params d$ on parameters). This technique projects each region in LA county to a representative point in the parameter space of $\varphi$. Next, we use Agglomerative Hierarchical Clustering to cluster the regions into multiple groups with respect to $\varphi$. We use the learned parameter valuations $(c, d)$ as features for clustering algorithm with the number of clusters as $k=3$. The results are illustrated in Fig.~\ref{fig:covid_example1}.

The \textit{red} cluster consists of the unsafe region or hot spots since it has small $d$ and large $c$. Each point in Fig.~\ref{fig:covid_exampl1_clustering} is associated with a region in LA County illustrated in Fig.~\ref{fig:covid_exampl1_map}. The total time for learning the clusters from spatio-temporal data is 813.65 seconds. Next, we are interested in learning an interpretable STREL formula for each cluster using our Decision Tree based approach. The result of applying the Decision Tree classifier on the learned clusters is shown as dash-lines in Fig.~\ref{fig:covid_exampl1_clustering}. To do a perfect separation, the Decision Tree learns 11 hyper-boxes which some of them only contain a negligible number of data points. While using all the 11 hyper-boxes results in guaranteed STREL formulas for each cluster of points, the STREL formulas that are learned might be complicated and hence not interpretable. To mitigate this issue and improve the interpretability of the learned STREL formulas for each cluster, we try to learn an approximate STREL formula for each cluster of points using the Decision Tree pruning method proposed in Algo.~\ref{alg:dt_prune}.

We compute an STREL formula for each cluster as follows:

\begin{align*}
& \varphi_{blue} = \varphi(0,15000) \wedge \neg \varphi(0,4691.29)\wedge \neg \varphi(5611.5,15000) = \neg \varphi(0,4691.29)\wedge \\ &\neg \varphi(5611.5,15000)\\ & \varphi_{green} = \varphi(0,4691.29) \wedge \neg \varphi(0,0) \wedge \neg \varphi(3180,4691.29) = \varphi(0,4691.29) \wedge \\ &\neg \varphi(3180,4691.29)\\ &\varphi^1_{red} = \varphi(3180,4691.29) \wedge \neg \varphi(8000,4691.29) \wedge \neg \varphi(3180,0) = \varphi(3180,4691.29)\\
&\varphi^2_{red} = \varphi(5611.5,15000) \wedge \neg \varphi(8000,15000) \wedge \neg \varphi(5611.5,4691.29) = \\ &\varphi(5611.5,15000)\\
&\varphi{red} = \varphi^1_{red} \vee \varphi^2_{red}.
\end{align*}

By replacing $\varphi(c,d)$ with $\somewhere{[0,d]}{}\{\F_{[0,10]}(x > c)\}$ in the above STREL formulas we get:

\begin{align*}
& \varphi_{blue} =\everywhere{[0,4691.29]}{}\{G_{[0,10]}(x = 0)\}\wedge \everywhere{[0,15000]}{}\{G_{[0,10]}(x < 5611.5)\}\\ & \varphi_{green} = \somewhere{[0,4691.29]}{}\{\F_{[0,10]}(x > 0)\} \wedge \everywhere{[0,4691.29]}{}\{G_{[0,10]}(x < 3180)\}\\ &\varphi^1_{red} = \somewhere{[0,4691.29]}{}\{\F_{[0,10]}(x > 3180)\}\\
&\varphi^2_{red} = \somewhere{[0,15000]}{}\{\F_{[0,10]}(x > 5611.5)\}\\
&\varphi_{red} = \varphi^1_{red} \vee \varphi^2_{red} = \somewhere{[0,4691.29]}{}\{\F_{[0,10]}(x > 3180)\} \vee \\ &\somewhere{[0,15000]}{}\{\F_{[0,10]}(x > 5611.5)\}
\end{align*}

The learned formula for the \textit{blue} cluster is $\varphi_{blue} =\everywhere{[0,4691.29]}{}\{G_{[0,10]}(x = 0)\}\wedge \everywhere{[0,15000]}{}\{G_{[0,10]}(x < 5611.5)\}$, which means that within the radius of 4691.29 meters around the \textit{blue} points, in the next 10 days, there is no new positive cases. By increasing the radius to 15000 meters there is new positive cases but the number does not exceed 5611. The learned formula for the \textit{green} cluster is $\varphi_{green} = \somewhere{[0,4691.29]}{}\{\F_{[0,10]}(x > 0)\} \wedge \everywhere{[0,4691.29]}{}\{G_{[0,10]}(x < 3180)\}$, which means that within the radius of 4691.29 meters from the \textit{green} points, at some point in the next 10 days, there will be new positive cases but the number will not exceed 3180. Finally, formula $\varphi_{red} = \somewhere{[0,4691.29]}{}\{\F_{[0,10]}(x > 3180)\} \vee \somewhere{[0,15000]}{}\{\F_{[0,10]}(x > 5611.5)\}$ which is learned for the \textit{red} cluster means that within the radius of 4691.29 meters from the \textit{red} points, in the next 10 days, the number of new positive cases exceeds 3180. By increasing the radius to 15000, in some region within this radius, the number of new positive cases goes beyond 5611. Thus, the regions associated with the \textit{red} points (illustrated in Fig.~\ref{fig:covid_exampl1_map}) are the hot spots and people should avoid any unnecessary communications in these regions in the next 10 days. The COVID-19 data that we considered is for September 2020. In Fig.~\ref{fig:cluster_changing_time}, we show the results of clustering for 3 different months in 2020, which confirms the rapid spread of the COVID-19 virus in LA county from April 2020 to September 2020.

\begin{figure}[!tp]
\centering
\setkeys{Gin}{height=3cm} 

\makebox[\textwidth]{%
  \setlength{\tabcolsep}{3pt}%
  \begin{tabular}{@{}cc@{}}
    \subfloat[Clusters learned from the COVID-19 data for April 2020.]{\includegraphics[width = .3\textwidth]{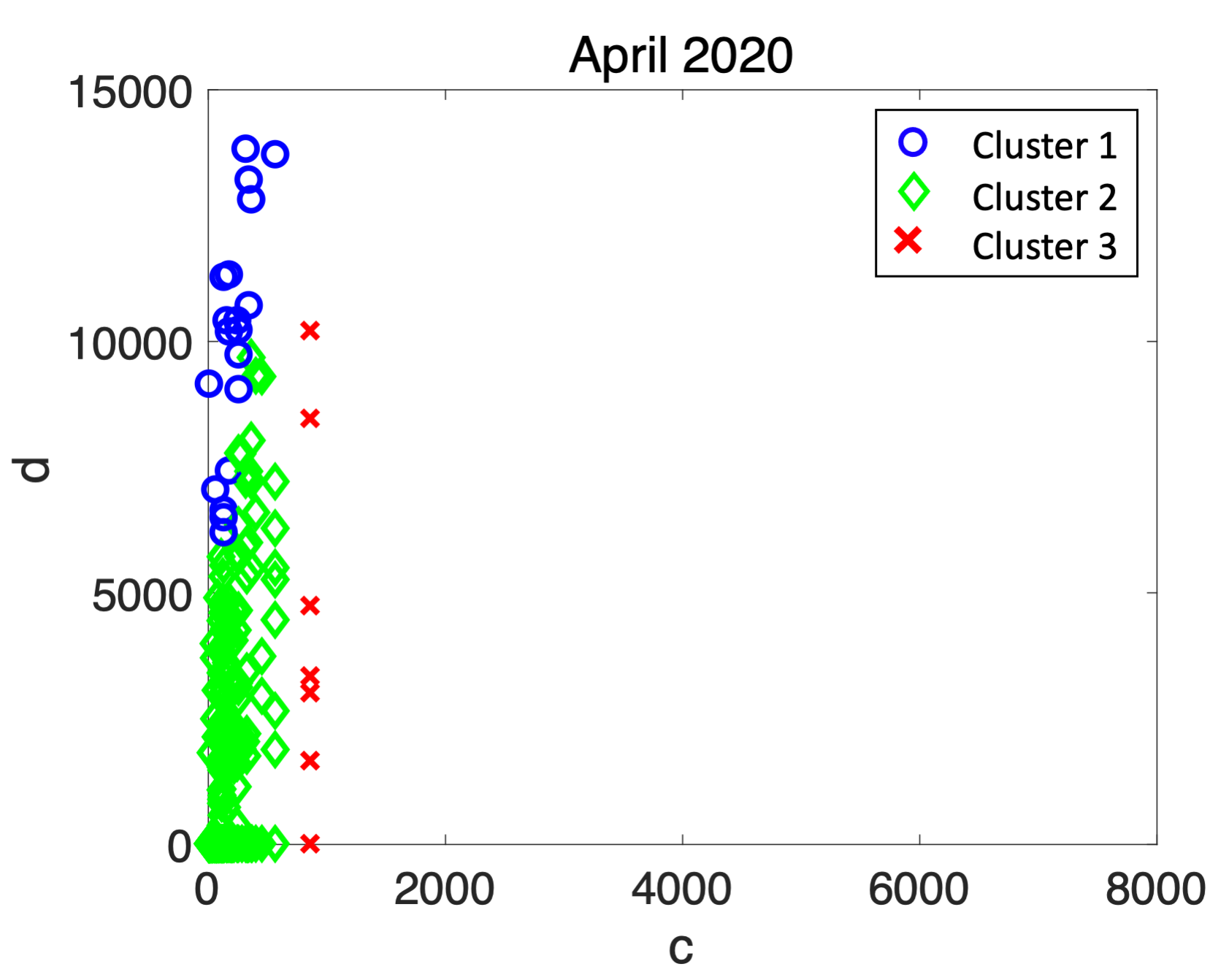}} \quad
  \subfloat[Clusters learned from the COVID-19 data for June 2020.]{\includegraphics[width = .3\textwidth]{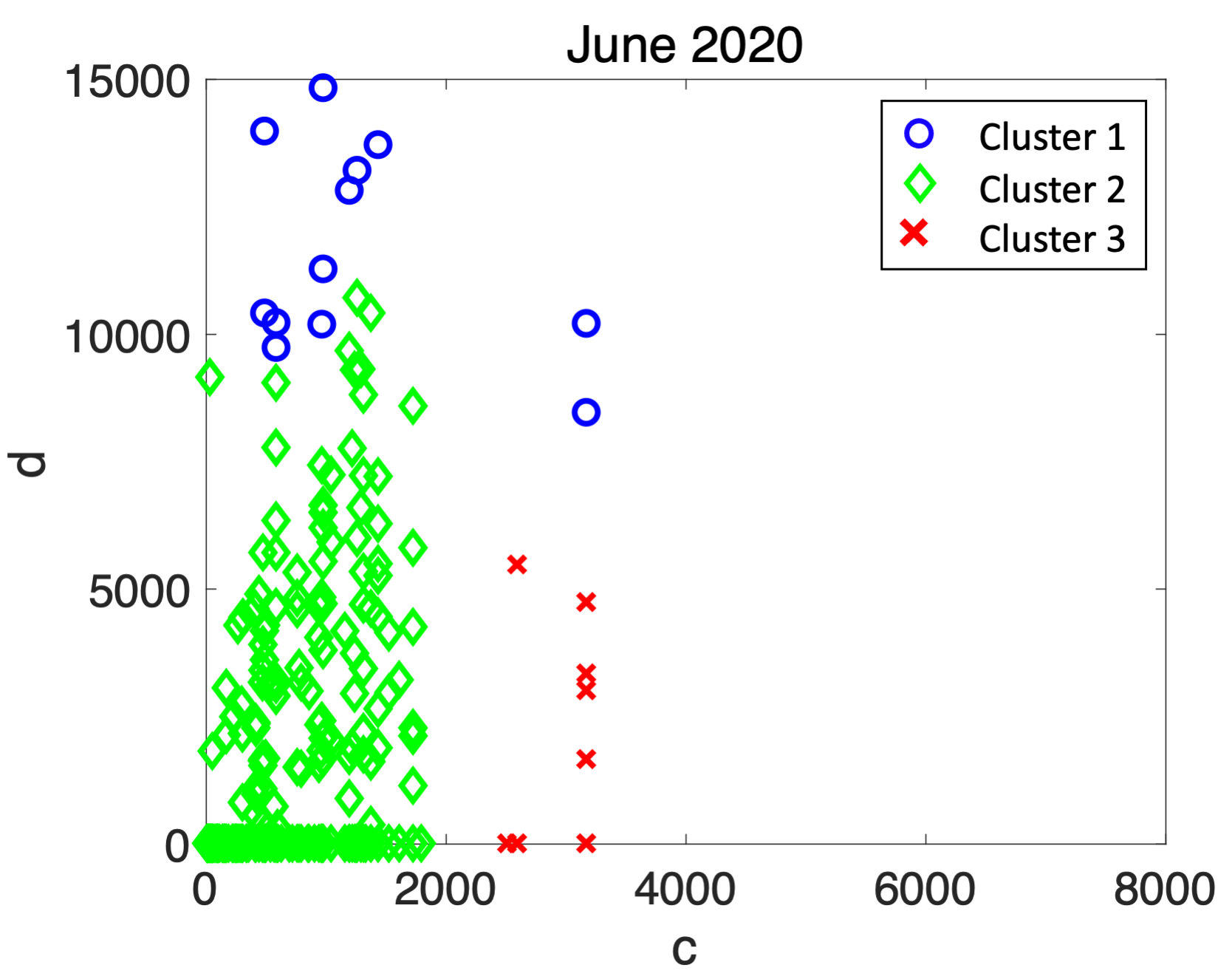}} \quad
  \subfloat[Clusters learned from the COVID-19 data for September 2020.]{\includegraphics[width = .3\textwidth]{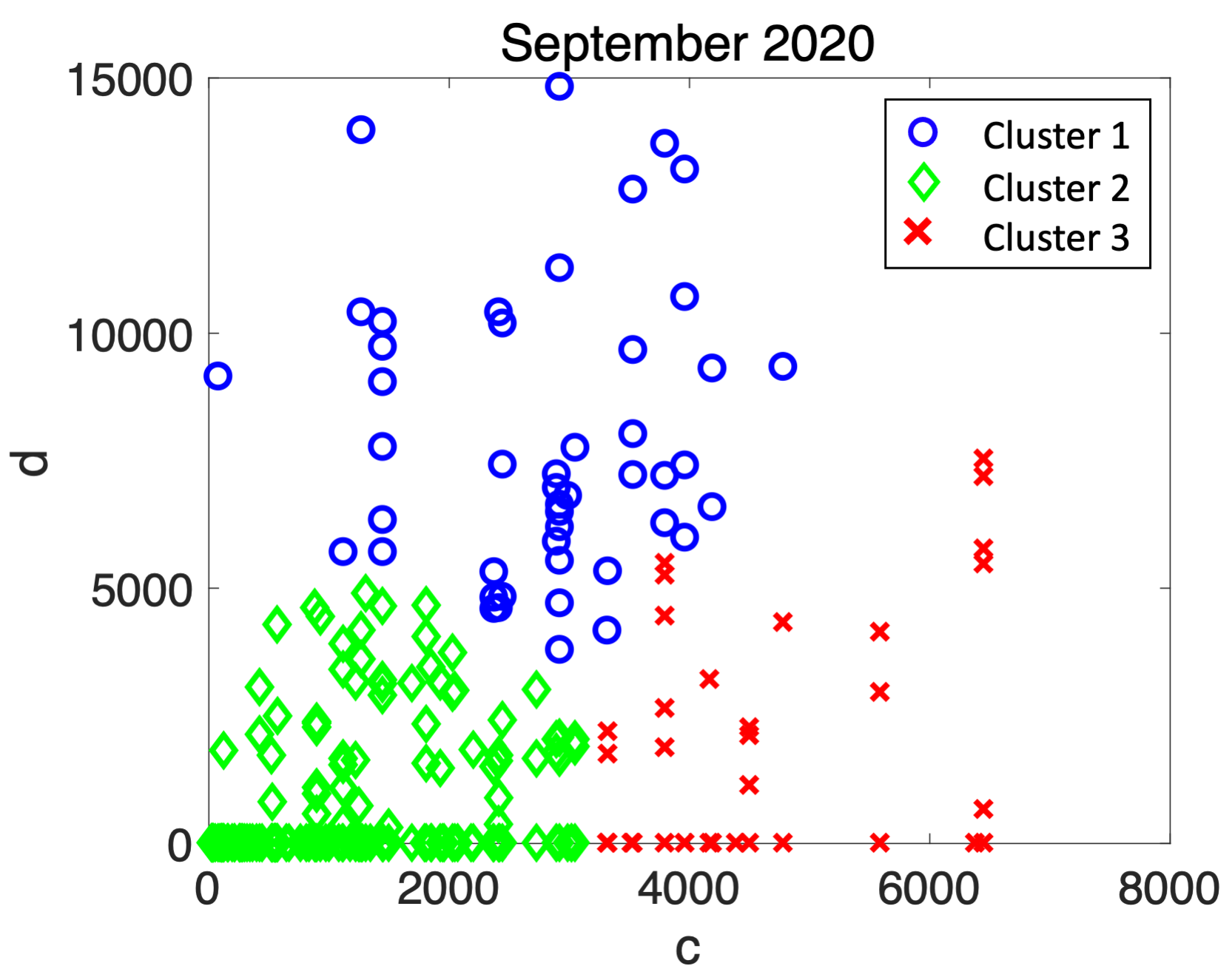}}\\
    \subfloat[Regions in LA county associated with the learned clusters from the COVID-19 data for April 2020.]{\includegraphics[width = .3\textwidth]{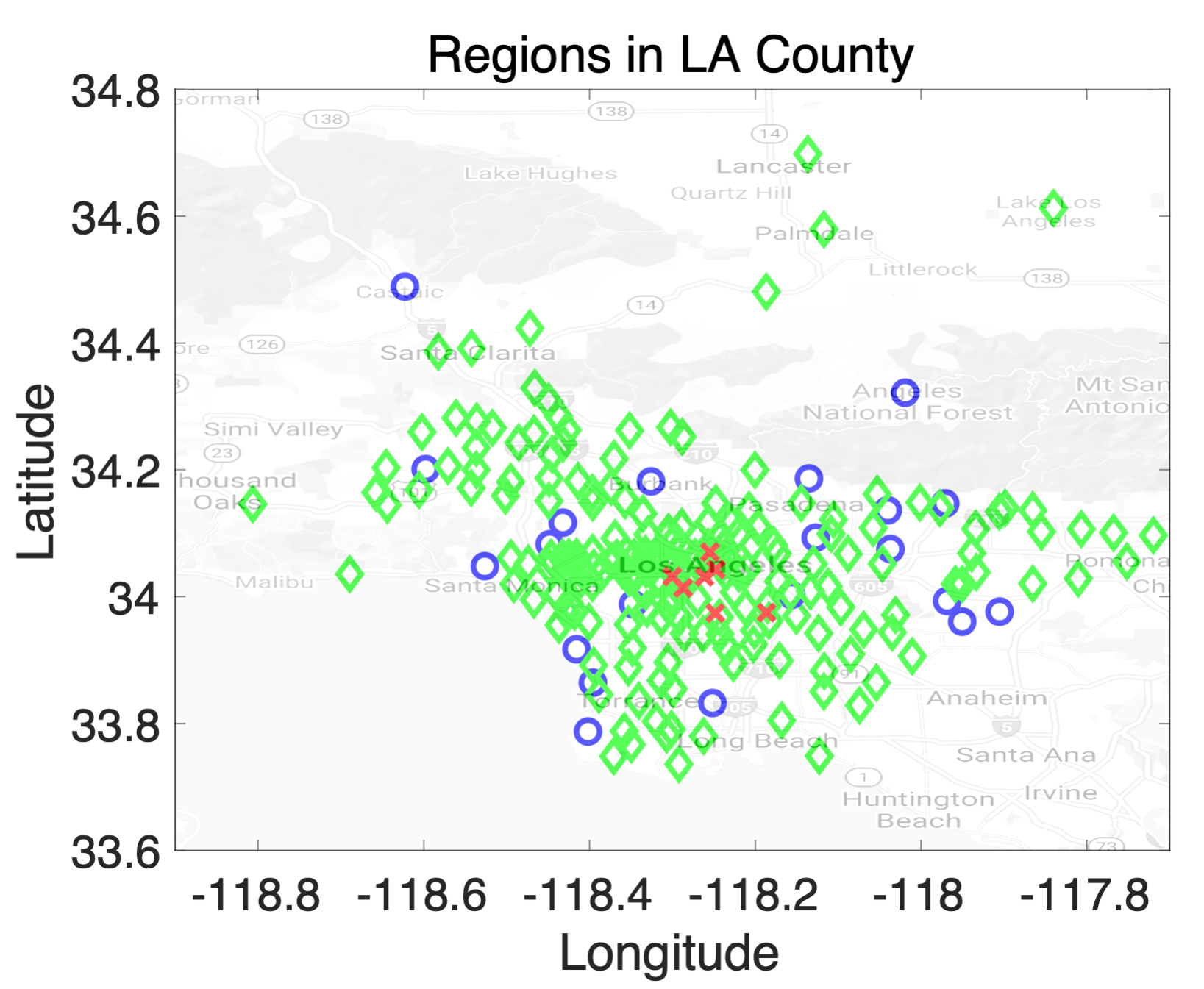}}
 \quad
    \subfloat[Regions in LA county associated with the learned clusters from the COVID-19 data for June 2020.]{\includegraphics[width = .3\textwidth]{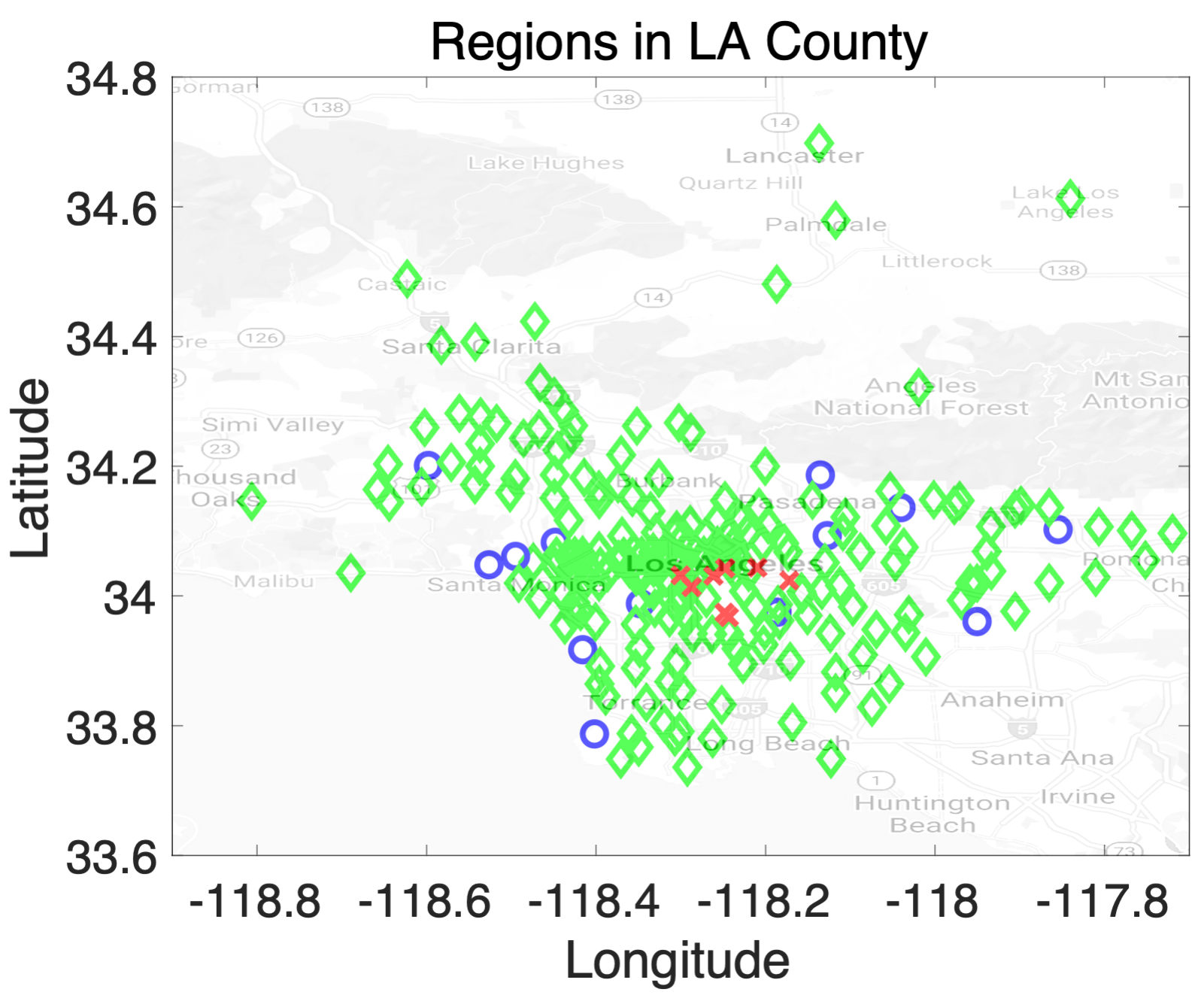}}
 \quad
     \subfloat[Regions in LA county associated with the learned clusters from the COVID-19 data for September 2020.]{\includegraphics[width = .3\textwidth]{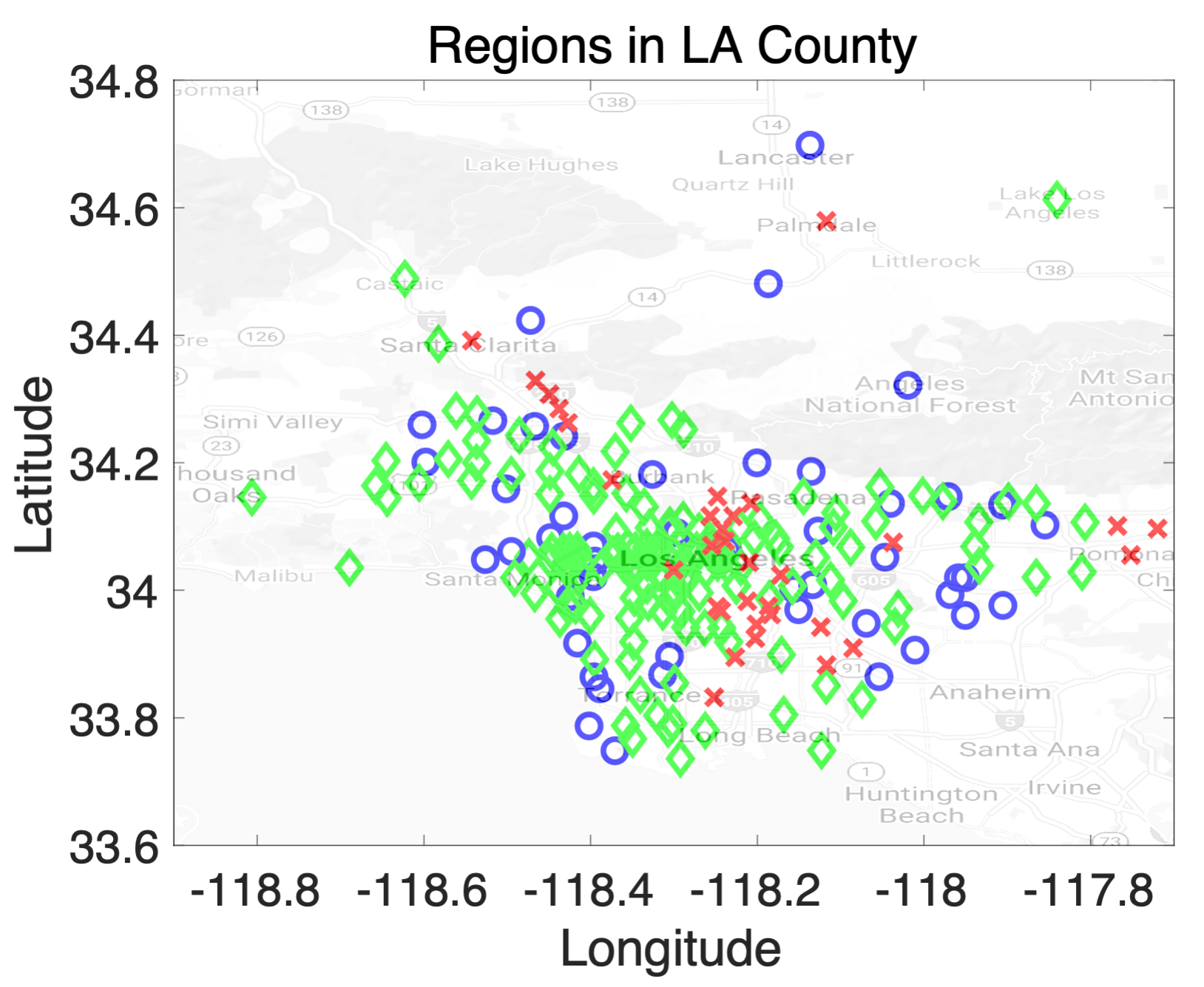}}
  \end{tabular}%

}
\caption{Changing of clustering with respect to time for the COVID-19 data and PSTREL formula $\varphi(c,d)$ with $\somewhere{[0,d]}{}\{\F_{[0,10]}(x > c)\}$. The plots confirm the rapid spread of the COVID-19 virus in LA county from April 2020 to September 2020.} 
\label{fig:cluster_changing_time}
\end{figure}

\subsection{BSS data from the city of Edinburgh}
Bike-Sharing System is an affordable and popular form of transportation that has been introduced to many cities such as the city of Edinburgh in recent years \cite{kreikemeyer2020probing}. Each BSS consists of a number of bike stations, distributed over a geographic area, and each station has a fixed number of bike slots. The users can pick up a bike, use it for a while, and then return it to the same or another station in the area. It is important for a BSS to satisfy the demand of its users. For example, ``within a short time, there should always be a bike and empty slot available in each station or its adjacent stations''. Otherwise, the user either has to wait a long time in the station or walk a long way to another far away station that has bike/slot availability. Understanding the behavior of a BSS can help users to decide on which station to use and also help operators to choose the optimal number of bikes/empty slots for each station. In this work we use STREL formulas to understand the behavior of BSS. Particularly, we consider the BSS in the city of Edinburgh which consist of 61 stations (excluding the stations with more than $15\%$ missing value). We construct the spatial model using the $(\alpha,\dhaver)$-Enhanced MSG approach with $\alpha=2$ resulting in a total of 91 edges in the spatial model.
A BSS should have at least one of two important properties to satisfy the demand of users. The first property is: ``a bike/empty slot should be available in the station within a short amount of time''. The second property is ``if there is no bike/empty slot available in a station, there should be a bike/empty slot available in the nearby stations''. The first property can be described using PSTREL formula $\varphi_{wait}(\tau) = (\F_{[0,\tau]}Bikes \geq 1) \wedge (\F_{[0,\tau]}Slots \geq 1)$, which means that at some time within the next $\tau$ seconds, there will be at least one bike/empty slot available in the station. Stations with small values of $\tau$ have short wait time which is desirable for the users, and stations with large values of $\tau$ have long wait times which is undesirable. For the second case, the users might prefer to walk a short distance to nearby stations instead of waiting a long time in the current station for a bike/slot to be available; this is related to the second important property of a BSS. The associated PSTREL formula for the second property is $\varphi_{walk}(d)=(\somewhere{[0,d]}{}Bikes \geq 1) \wedge (\somewhere{[0,d]}{}Slots \geq 1)$, which means that ``at some station within the radius d of the current station, there is at least one bike/empty slot available''. If the value of $d$ for a station is large, this means that the user should walk a long way to far stations to find a bike/empty slot available. 
We combine the two PSTREL formulas into $\varphi(\tau, d) = G_{[0,3]}\{\varphi_{wait}(\tau) \vee \varphi_{walk}(d)\}$, which means that always within the next 3 hours at least one of the properties $\varphi_{wait}(\tau)$ or $\varphi_{walk}(d)$ should hold. We try to learn the tight value of $\tau$ and $d$ for each station.  Stations with small values of $\tau$ and $d$ are desirable stations. We first apply Algo.~\ref{alg:bisection_search} with the order $d >_\params \tau$ to learn the tight parameter valuations, and then apply our clustering followed by a Decision Tree algorithm to learn separating hyper-boxes for the clusters. The results are illustrated in Fig.~\ref{fig:bss_clustering_runex}.

{\em Green} points that have small values of $\tau$ and $d$ are desirable stations. The {\em orange} point is associated with a station with a long wait time (around $35$ minutes). The {\em red} points are the most undesirable stations as they have a long wait time and do not have nearby stations that have bike/empty slots availability. Fig.~\ref{fig:bss_map_runex} shows the location of each station on map which confirms that the stations associated with the {\em red} points are far from other stations. That's the reason for having a large value of $d$. The time that takes to learn the clusters is 681.78 seconds. Next, we try to learn an interpretable STREL formula for each cluster based on the monotonicity direction of $\tau$ and $d$ as follows:

\begin{align*}
&\varphi_{green} = \varphi(17.09,2100) \wedge \neg \varphi(0,2100) \wedge \neg \varphi(17.09,0) = \varphi(17.09,2100)\\ &\varphi_{orange} = \varphi(50,1000.98) \wedge \neg \varphi(17.09,1000.98) \wedge \varphi(50,0) = \varphi(50,1000.98) \wedge \\ &\neg \varphi(17.09,1000.98) \\& \varphi_{red} = \varphi(50, 2100) \wedge \neg \varphi(17.09, 2100) \wedge \neg \varphi(50, 1000.98) = \neg \varphi(17.09, 2100) \wedge \\ &\neg \varphi(50, 1000.98)
\end{align*}

By replacing $\varphi(\tau, d)$ with $G_{[0,3]}\{\varphi_{wait}(\tau) \vee \varphi_{walk}(d)\}$ in the above STREL formulas we get:

\begin{align*}
&\varphi_{green} = G_{[0,3]}\{\varphi_{wait}(17.09) \vee \varphi_{walk}(2100)\}\\ &\varphi_{orange} = G_{[0,3]}\{\varphi_{wait}(50) \vee \varphi_{walk}(1000.98)\} \wedge \\ &\neg G_{[0,3]}\{\varphi_{wait}(17.09) \vee \varphi_{walk}(1000.98)\}\\ &\varphi_{red} = \neg G_{[0,3]}\{\varphi_{wait}(17.09) \vee  \varphi_{walk}(2100)\} \wedge \\ &\neg G_{[0,3]}\{\varphi_{wait}(50) \vee \varphi_{walk}(1000.98)\}
\end{align*}

The intuition behind the learned STREL formula for the green cluster is that always within the next 3 hours the wait time for bike/slot availability is less than 17.09 minutes, or the walking distance to the nearby stations with bike/slot availability is less than 2100 meters. Fig.~\ref{fig:bss_clustering_runex} illustrates that the actual walking distance for green points is less than or equal to 1000.98 meters, and the reason for learning 2100 meters is that the Decision Tree tries to learn robust and relaxed boundaries for each class. The STREL formula $\varphi_{orange}$ means that for the next 3 hours at least one of the properties $\varphi_{wait}(50)$ or $\varphi_{walk}(1000.98)$ holds for the orange stations. However, for a smaller wait time equal to 17.09 seconds, at least once in the next 3 hours, both the properties $\varphi_{wait}(17.09)$ and $\varphi_{walk}(1000.98)$ do not hold. The intuition behind the learned STREL formula for the orange stations is that, the orange stations have long wait time because they satisfy the property $G_{[0,3]}\{\varphi_{wait}(50) \vee \varphi_{walk}(1000.98)\}$ and falsify the property $G_{[0,3]}\{\varphi_{wait}(17.09) \vee \varphi_{walk}(1000.98)\}$. The red points falsify the property $G_{[0,3]}\{\varphi_{wait}(17.09) \vee  \varphi_{walk}(2100)\}$ which is associated with a short wait time and the property $G_{[0,3]}\{\varphi_{wait}(50) \vee \varphi_{walk}(1000.98)\}$ which is associated with the short walking distance. Therefore, the red points are the most undesirable stations.

% \begin{figure}[!tp]
% \centering
% \setkeys{Gin}{height=4cm} 

% \makebox[\textwidth]{%
%   \setlength{\tabcolsep}{3pt}%
%   \begin{tabular}{@{}cc@{}}
%     \subfloat[Hyper-boxes learned from BSS data. \label{fig:bss_exampl1_clustering}]{\includegraphics[width = .4\textwidth]{figs/BSS/exampl1/bss_exampl_clustering.png}} \quad
%   \subfloat[Red points: far away stations. Orange points: long wait time. \label{fig:bss_exampl1_map}]{\includegraphics[width = .4\textwidth]{figs/BSS/exampl1/bss_exampl_regions.png}} 
% %       \quad
% %       \subfloat[Hyper-boxes learned from the Food Court data.\label{fig:building_clustering}]{\includegraphics[width = .2\textwidth]{figs/building/building_clustering.png}} \quad
% %   \subfloat[Locations in the Food Court associated with the learned clusters.\label{fig:building_regions}]{\includegraphics[width = .2\textwidth]{figs/building/building_regions.png}}
%   \end{tabular}%
% }
% \caption{} 
% \label{fig:bss_example1}
% \end{figure}

\mypara{Comparison with Traditional ML approaches}To compare our framework with traditional ML approaches, we apply KMeans clustering from tslearn library \cite{JMLR:v21:20-091}, which is a library in Python for analysis of time-series data, on the BSS data which uses DTW for similarity metric. The results of clustering are represented in Fig.~\ref{fig:bss_kmeans}. We also apply k-Shape clustering from tslearn library on the same dataset, where the results are illustrated in Fig.~\ref{fig:bss_kshape}. k-Shape is a partitional clustering algorithm that preserves the shapes of time series. While KMeans and k-Shape clustering approaches perform faster than our approach (both approaches run in less than 5 seconds), the artifacts trained by these approaches often lack interpretability despite our approach (for example, identifying stations with long wait time or far away stations). Furthermore, since our learning approach for each node/location is independent of other locations, it is possible to improve the run-time by parallelization.

\begin{figure}[!tp]
\centering
\setkeys{Gin}{height=3cm} 

\makebox[\textwidth]{%
  \setlength{\tabcolsep}{3pt}%
  \begin{tabular}{@{}cc@{}}
    \subfloat[Clusters learned from BSS data \label{fig:bss_kmeans_slots}]{\includegraphics[width = .3\textwidth]{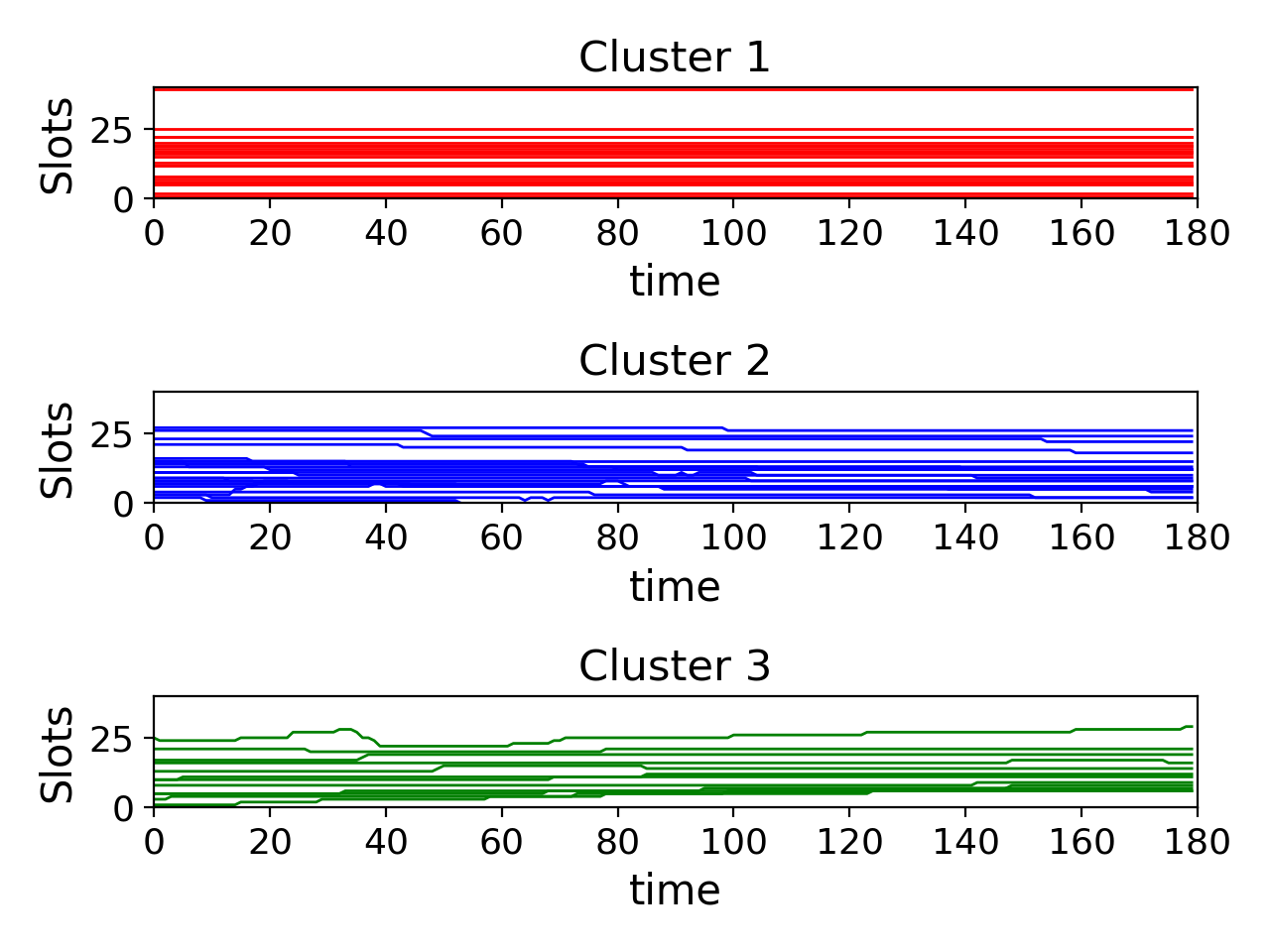}} \quad
  \subfloat[Clusters learned from BSS data \label{fig:bss_kmeans_bikes}]{\includegraphics[width = .3\textwidth]{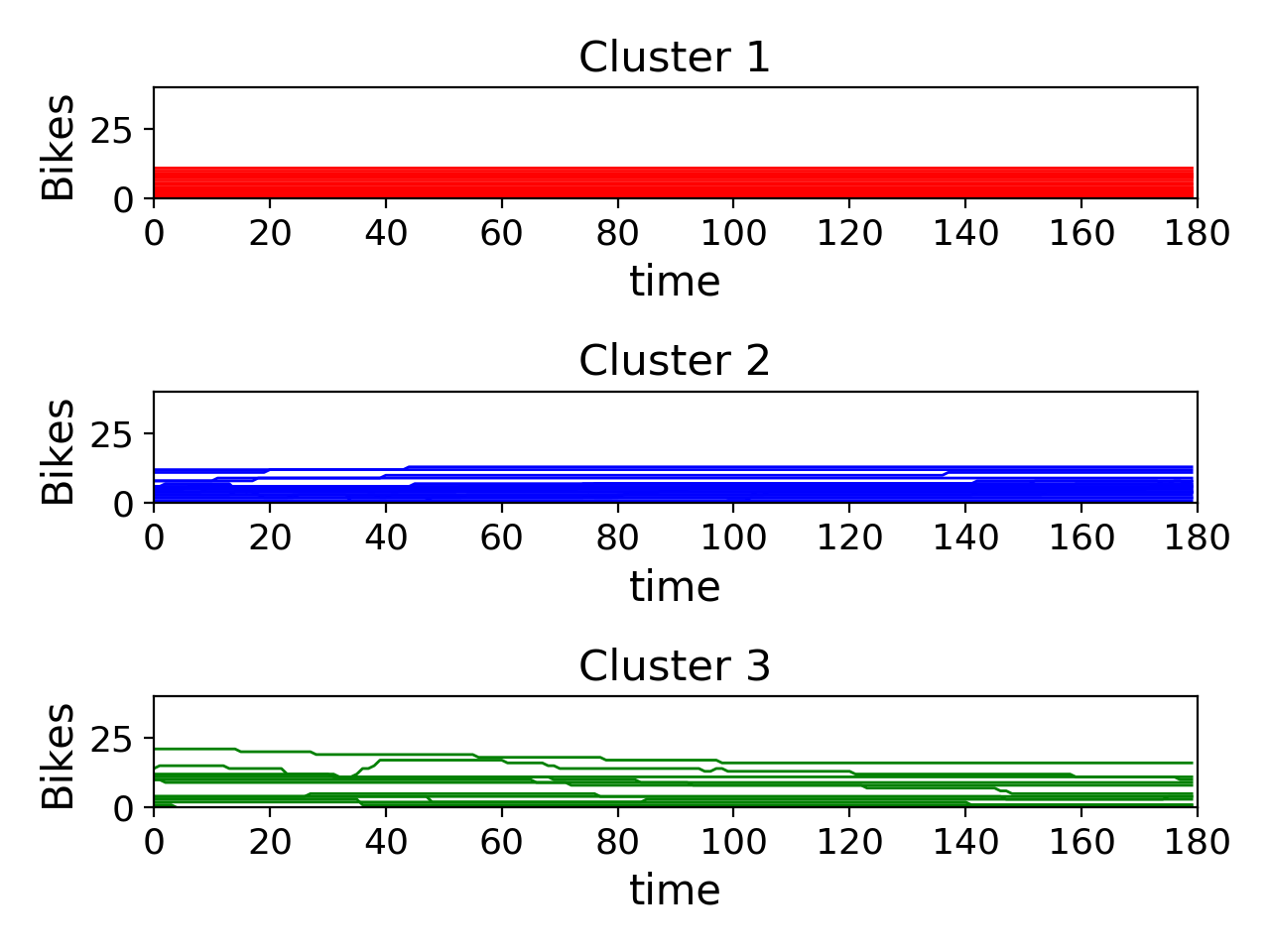}} \quad
  \subfloat[Regions in Edinburgh city associated with the learned clusters.  \label{fig:bss_kmeans_regions}]{\includegraphics[width = .3\textwidth]{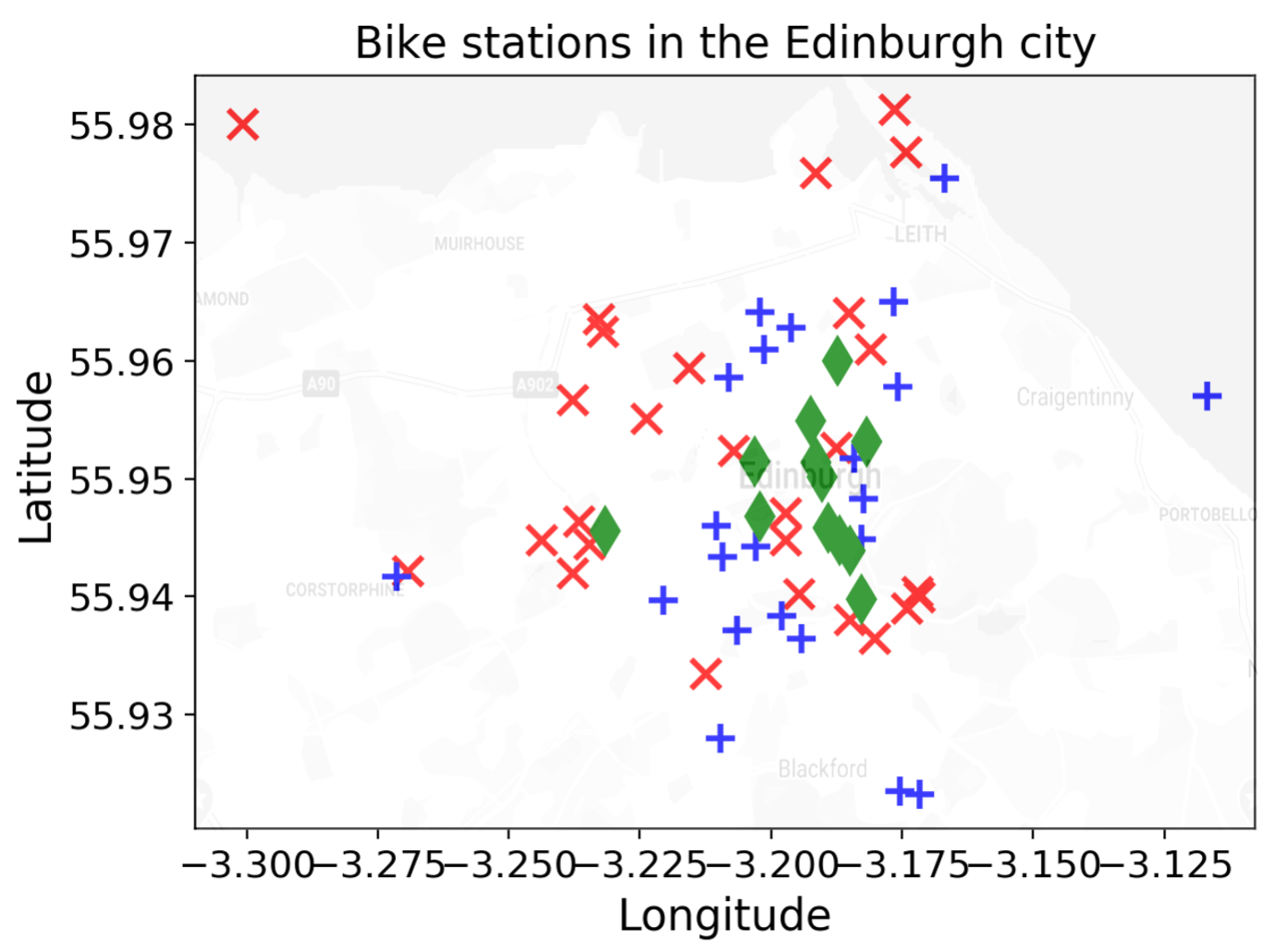}}

  \end{tabular}%
}
\caption{Using KMeans approach from tslearn library with DTW metric to cluster BSS spatio-temporal traces.} 
\label{fig:bss_kmeans}
\end{figure}

\begin{figure}[!tp]
\centering
\setkeys{Gin}{height=3cm} 

\makebox[\textwidth]{%
  \setlength{\tabcolsep}{3pt}%
  \begin{tabular}{@{}cc@{}}
    \subfloat[Clusters learned from BSS data \label{fig:bss_kshape_slots}]{\includegraphics[width = .3\textwidth]{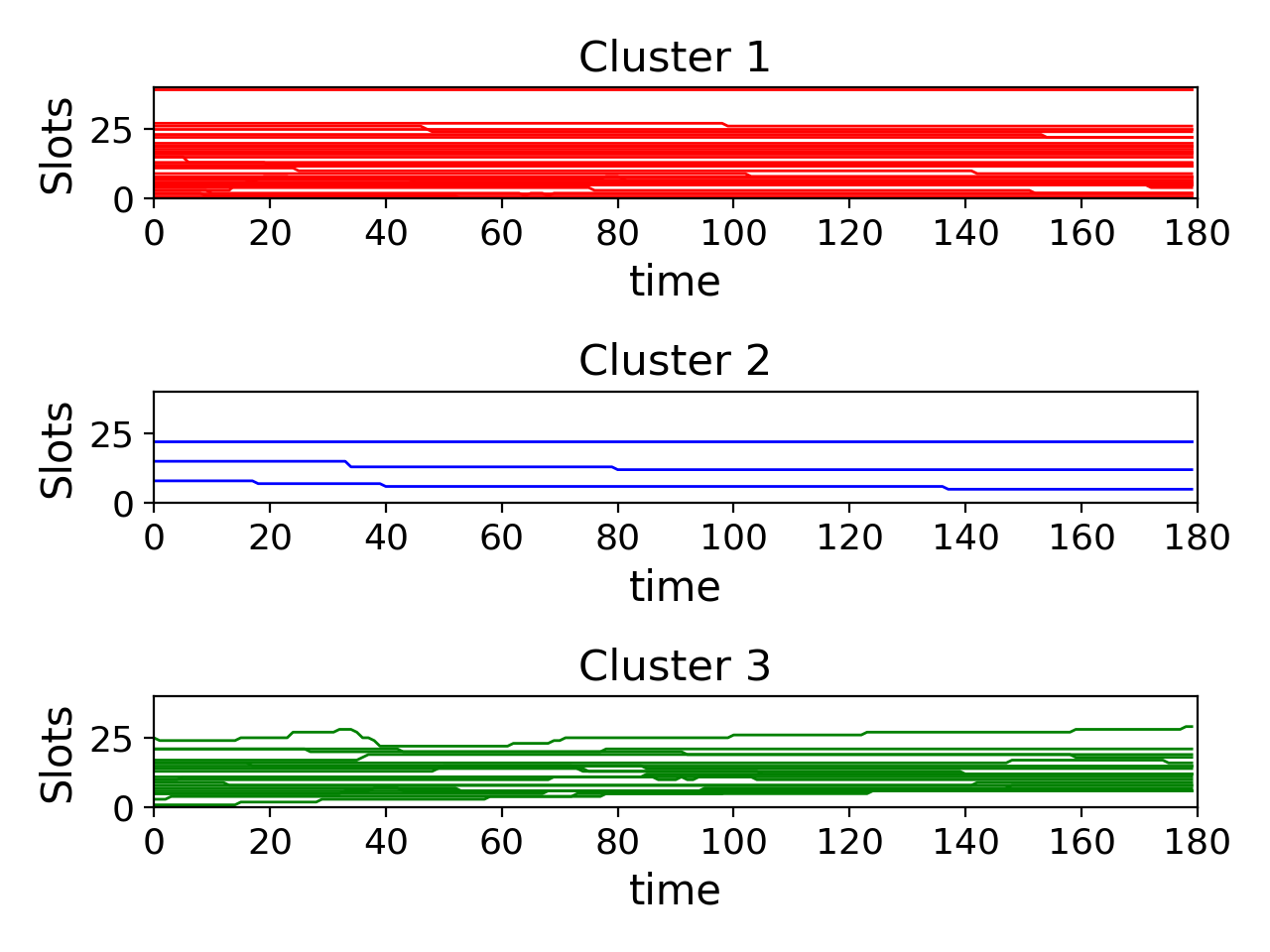}} \quad
  \subfloat[Clusters learned from BSS data \label{fig:bss_kshape_bikes}]{\includegraphics[width = .3\textwidth]{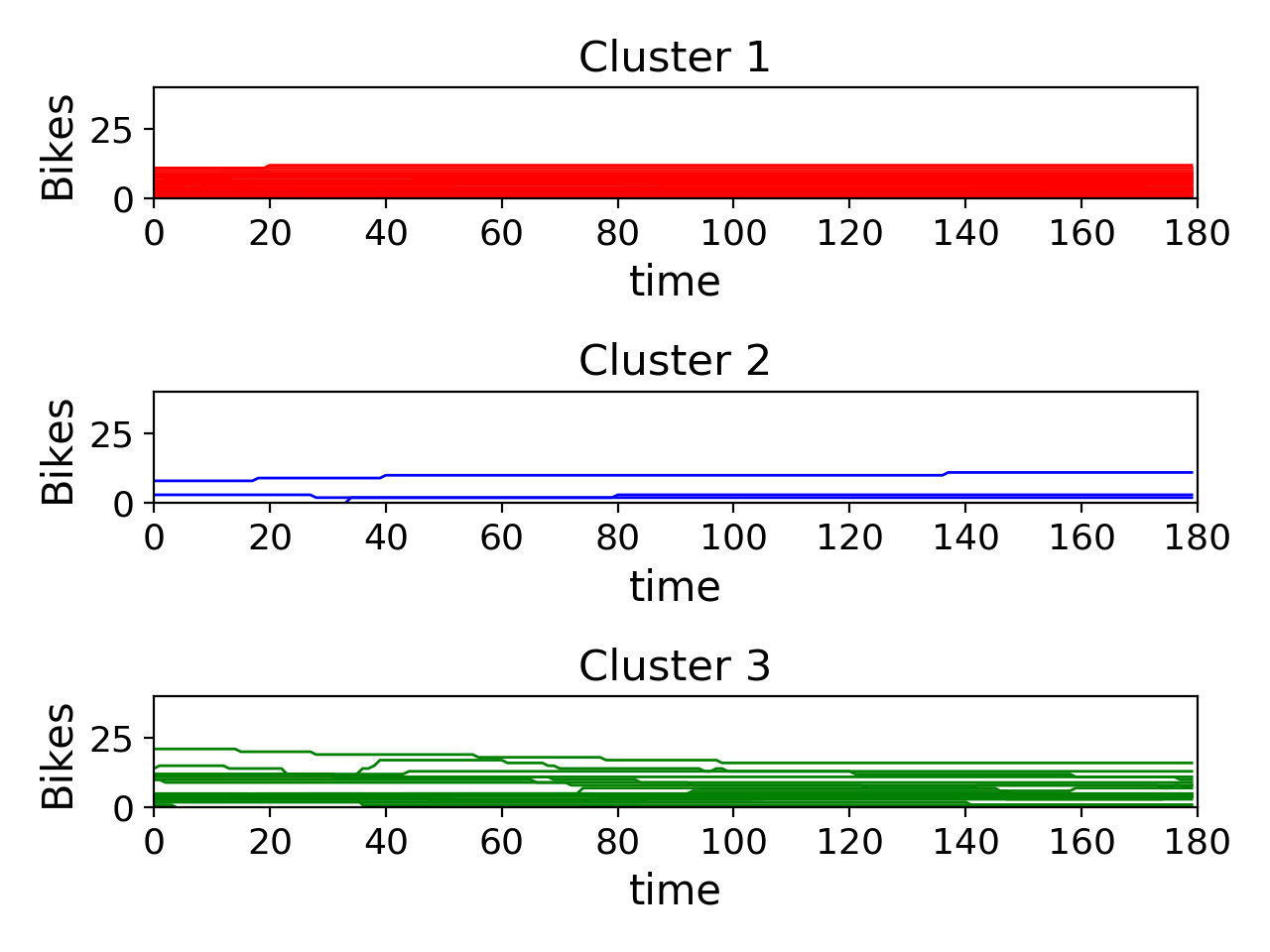}} \quad
  \subfloat[Regions in Edinburgh city associated with the learned clusters.  \label{fig:bss_kshape_regions}]{\includegraphics[width = .3\textwidth]{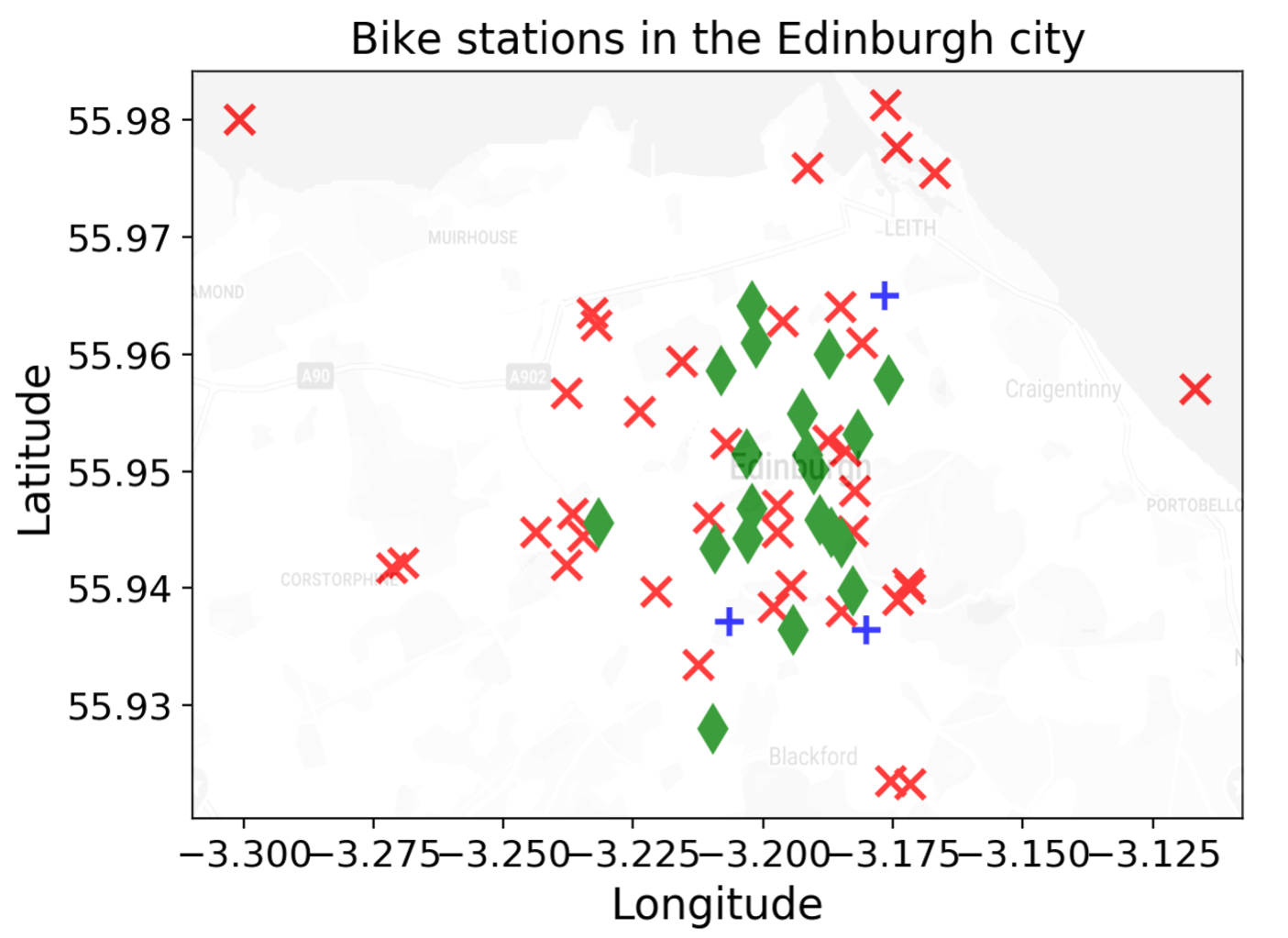}}

  \end{tabular}%
}
\caption{Using Kshape approach from tslearn library to cluster BSS spatio-temporal traces.} 
\label{fig:bss_kshape}
\end{figure}

\subsection{Outdoor Air Quality data from California}
Air pollution is one of the most important emerging concerns in environmental studies which can result in global warming and climate change. Most importantly, air pollution was ranked among the top leading cause of death in the world \cite{gunawan2018design}. For example, particulate matters smaller than $2.5 \mu m$ (PM2.5) can deeply penetrate into lung and impair its functions \cite{xing2016impact}. As a result, scientist are using the concentration of PM2.5 in air as the most important air quality index. Therefore, understanding the pattern of air pollutants like PM2.5 in different regions of a city is of significant importance because it can help people to take the necessary precautions such as preventing unnecessary communications in pollutant areas. In this work, we consider the Air Quality data from California that has been gathered by the United States Environmental Protection Agency (US EPA). Among reported pollutants we focus on PM2.5 contaminant, and try to learn the patterns in the amount of PM2.5 in the air using STREL formulas. The standard levels of PM2.5 are as follows: (1) $0 \leq PM2.5 \leq 12.0 \mu g/m^3$ is considered as good, (2) $12.1 \leq PM2.5 \leq 35.4 \mu g/m^3$: moderate, (3) $35.5 \leq PM2.5 \leq 55.4 \mu g/m^3$: unhealthy for sensitive groups, and (4) $55.5 \leq PM2.5$ is considered as unhealthy. To create the spatial model, we use the $(\alpha,d)$-Enhanced MSG Spatial Model approach with $\alpha=2$, resulting in a total of 160 edges.
We consider the PSTREL formula $\varphi(c,d) = G_{[0,10]}\{\escape{[d,16000]}{}(PM2.5 < c)\}$, and try to find the tight values of $c$ and $d$ for each region in California. A location $\ell$ will satisfy this property if it is always true within the next 10 days, that there exists a location $\ell'$ at a distance more than $d$, and a route $\tau$ starting from $\ell$ and reaching $\ell'$ such that all the locations in the route satisfy the property $PM2.5 < c$. Hence, it might be possible to escape at a distance greater than $d$ always satisfying property $PM2.5 < c$. 
%$16000$ is the maximum distance of the graph so it corresponds to 
The result of applying our learning method on Air Quality data with respect to this PSTREL formula is illustrated in Fig.~\ref{fig:pollution_exampl2_clustering}. The run-time of the algorithm is 136.02 seconds. Cluster 8 is the best cluster as it has a small value of $c$ and large value of $d$ which means that there exists a long route from the point in cluster 8 with low density of PM2.5. Cluster 3 is the worst cluster as it has a large value of $c$ and a small value of $d$, which means that the density of $PM2.5$ in these regions is extremely high. We try to learn an STREL formula for clusters 8 and 3 as follows:

\begin{align*}
&\varphi_3 = \varphi(500,0) \wedge \neg \varphi(500,2500) \wedge \neg \varphi(216,0) = G_{[0,10]}\{\escape{[0,16000]}{}(PM2.5 < 500)\} \wedge \\& \neg G_{[0,10]}\{\escape{[2500,16000]}{}(PM2.5 < 500)\} \wedge \neg G_{[0,10]}\{\escape{[0,16000]}{}(PM2.5 < 216)\} 
% = 
% \\& G_{[0,10]}\{\escape{[0,16000]}{f}(PM2.5 < 500)\} \wedge  F_{[0,10]}\{\neg \escape{[2500,16000]}{f}(PM2.5 < 500)\} \wedge F_{[0,10]}(PM2.5 \geq 216)
\\ &
\varphi_8 = \varphi(83.5,12687.5) \wedge \neg \varphi(0,12687.5) \wedge \neg \varphi(83.5,15000) = \varphi(83.5,12687.5) = \\& G_{[0,10]}\{\escape{[12687.5,16000]}{}(PM2.5 < 83.5)\}
\end{align*}

The formula $\varphi_3$ holds in locations where, in the next 10 days, $PM2.5$ is always less than $500$, but at least in 1 day $PM2.5$ reaches $216$ and  at least in 1 day there is no safe paths reaching locations at the distance greater than 2500 where all the locations in the routes have still $PM2.5 < 500$.
The formula $\varphi_8$ holds in all the locations where it is  always true, within the next 10 days, that there exists a route reaching a location at a distance greater than 12687.5 m from the the current one, such that all the locations in the route have $PM2.5$ less than $83.5$

\subsection{Food Court Building}
\begin{figure}[!tp]
\centering
\setkeys{Gin}{height=4cm} 

\makebox[\textwidth]{%
  \setlength{\tabcolsep}{4pt}%
  \begin{tabular}{@{}cc@{}}
    \subfloat[Clusters learned from Food Court data.\label{fig:building_clustering}]{\includegraphics[width = .4\textwidth]{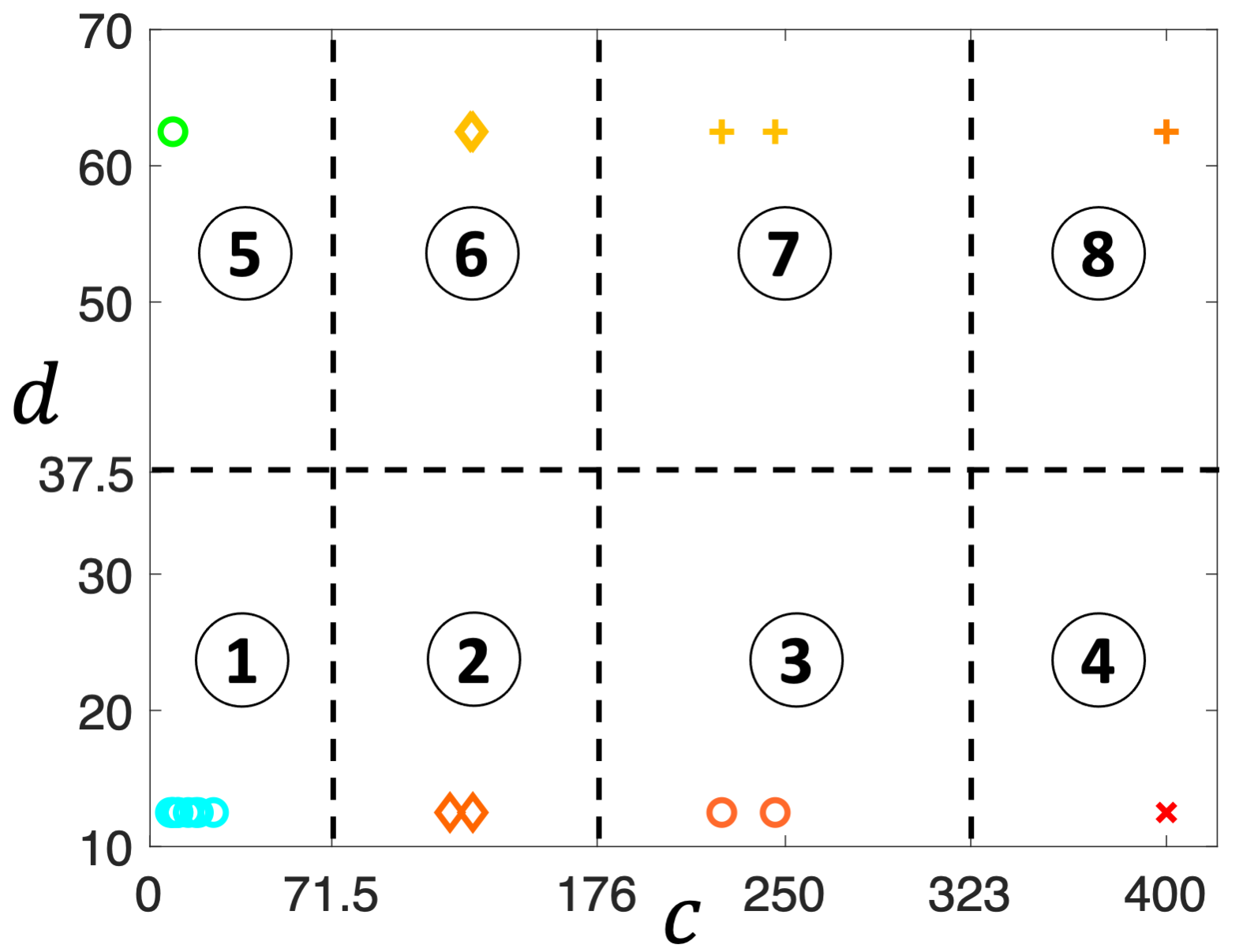}}\quad
  \subfloat[Locations in the Food Court associated with the learned clusters.\label{fig:building_regions}]{\includegraphics[width = .4\textwidth]{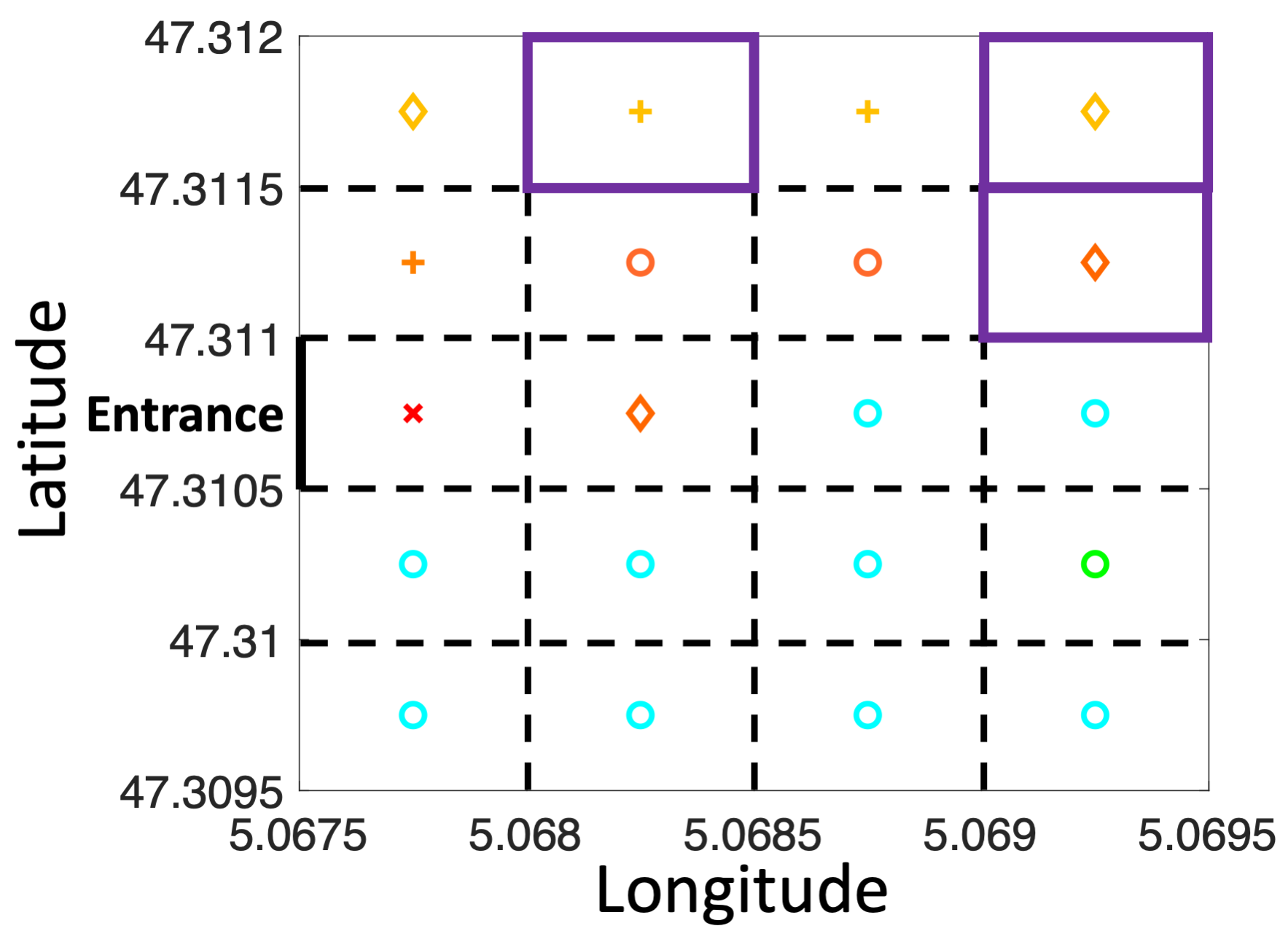}}
  \end{tabular}%
}
\caption{}
\label{fig:pollution_foodcourt}
\end{figure}

Here, we consider a synthetic dataset that simulates movements of customers in different areas of a food court. Due to COVID-19 pandemic, we are interested in identifying crowded areas in the food court, to help managers of the food court to have a better arrangement of different facilities (for example, keeping the popular restaurants in the food court separated) to prevent crowded area. To synthesize the dataset, we divide the food court into 20 regions, considering one of the regions as entrance, and three of them as popular restaurants. We create the dataset using the following steps: (1) we make a simplifying assumption that all of the customers enter the food court at time 0, which means that we set initial location of each customer as entrance (2) in every 10 minutes, we choose a random destination for each customer. The destination can be current location of the customer (in this case the customer does not go anywhere), one of the popular areas (we choose the probability as 0.8) or other areas of the food court. After choosing the destination for each customer, we simulate moving of the customer towards the destination with speed of $1.4 m/s$, which is the average speed of walking. To create the spatial model, we assume the centers of each of the 20 regions as a node and connect the 20 nodes using the $(\alpha,\dhaver)$-Enhanced MSG approach with $\alpha=2$ resulting in 35 edges in the graph.
The PSTREL formula $\varphi(c,d) = \somewhere{[0,d]}{}\{\F_{[0,3]}(numPeople > c)\}$ means that somewhere within the radius $d$ from a location, at least once in the next 3 hours, the number of people in the location exceeds the threshold $c$. Using this property, we can identify the crowded areas, and take necessary actions to mitigate the spread of COVID-19 in the Food Court area. We apply our method on the synthetic dataset that simulates the movements of customers in a Food Court, and try to learn the tight values of $d$ and $c$ for each location followed by clustering and our Decision Tree classification approach. The run-time of our learning approach is 78.24 seconds, and the results are illustrated in Fig.~\ref{fig:building_clustering}. Cluster 4 which has a large value of $c$ and a small value of $d$ is associated with the most crowded area. We show the locations in the Food Court associated with each cluster in Fig.~\ref{fig:building_regions}. The results show a larger value of $c$ for the entrance and  popular restaurants which confirm that these area are more crowded compared to other locations in the Food Court. Next, we try to learn an STREL formula for cluster 4 (associated with the most crowded location) and cluster 5 (associated with the most empty location) as follows:

\begin{align*}
&\varphi_4 = \varphi(323,37.5) \wedge \neg \varphi(323,10) \wedge \neg \varphi(420,37.5) = \varphi(323,37.5) = \\ &\somewhere{[0,37.5]}{}\{\F_{[0,3]}(numPeople > 323)\}\\ 
&\varphi_5 = \varphi(0,70) \wedge \neg \varphi(71.5,70) \wedge \neg \varphi(0,37.5) = \somewhere{[0,70]}{}\{\F_{[0,3]}(numPeople > 0)\} \wedge \\ &\everywhere{[0,70]}{}\{G_{[0,3]}(numPeople \leq 71.5)\} \wedge \everywhere{[0,37.5]}{}\{G_{[0,3]}(numPeople = 0)\}
\end{align*}

The STREL formula $\varphi_4$ means that somewhere within the radius 37.5 meters from the location, at least once within the next 3 hours, the number of people exceeds 323 which shows a crowded area. The learned formula for cluster 5 means that for the next 3 hours there will be no one in the radius 37.5 from the location. However, between the radius 37.5 and 70 from the location, there will be some people but the number does not exceed 72.
\label{sec:apendix}

\end{document}